\documentclass{article}
\usepackage{ijcai25}

\usepackage{times}
\usepackage{soul}
\usepackage{url}
\usepackage[hidelinks]{hyperref}
\usepackage[utf8]{inputenc}
\usepackage[small]{caption}
\usepackage{graphicx}
\usepackage{amsmath}
\usepackage{amsthm}
\usepackage{listings}
\usepackage{xcolor}

\usepackage{booktabs}
\usepackage{algorithm}
\usepackage[switch]{lineno}
\usepackage{natbib} 
\usepackage{multirow}
\usepackage{amssymb}
\usepackage{algpseudocode}
\usepackage{bm}
\usepackage{colortbl}
\newtheorem{theorem}{Theorem}

\title{System 1\&2 Synergy via Dynamic Model Interpolation}
\author{
    Chenxu Yang$^{1,2}$\footnotemark[1]~\and
    Qingyi Si$^3$\thanks{$\quad$ Equal Contribution}\and
    Chong Tian$^4$\and Xiyu Liu$^{1,2}$\and Dingyu Yao$^{1,2}$\and \\ Chuanyu Qin$^{1,2}$\and 
    Zheng Lin$^{1,2}$\thanks{$\quad$ Zheng Lin is the corresponding author.}\and Weiping Wang$^1$\and Jiaqi Wang$^3$\and
\affiliations
   \textsuperscript{\rm 1}Institute of Information Engineering, Chinese Academy of Sciences, Beijing, China \\
  \textsuperscript{\rm 2}School of Cyber Security, University of Chinese Academy of Sciences, Beijing, China \\
  \textsuperscript{\rm 3}JD.COM \textsuperscript{\rm 4}MBZUAI \\
\emails  
  \{yangchenxu,linzheng\}@iie.ac.cn
}

\begin{document}

\maketitle

\begin{abstract}
Training a unified language model that adapts between intuitive System 1 and deliberative System 2 remains challenging due to interference between their cognitive modes. Recent studies have thus pursued making System 2 models more efficient. However, these approaches focused on output control, limiting what models produce. We argue that this paradigm is misaligned: output length is merely a symptom of the model's cognitive configuration, not the root cause. In this work, we shift the focus to capability control, which modulates \textit{how models think} rather than \textit{what they produce}. To realize this, we leverage existing Instruct and Thinking checkpoints through dynamic parameter interpolation, without additional training. Our pilot study establishes that linear interpolation yields a convex, monotonic Pareto frontier, underpinned by representation continuity and structural connectivity. Building on this, we propose \textbf{DAMI} (\textbf{D}yn\textbf{A}mic \textbf{M}odel \textbf{I}nterpolation), a framework that estimates a query-specific Reasoning Intensity $\lambda(q)$ to configure cognitive depth. For training-based estimation, we develop a preference learning method encoding accuracy and efficiency criteria. For zero-shot deployment, we introduce a confidence-based method leveraging inter-model cognitive discrepancy. 
Experiments on five mathematical reasoning benchmarks demonstrate that DAMI achieves higher accuracy than the Thinking model while remaining efficient, effectively combining the efficiency of System 1 with the reasoning depth of System 2.

\end{abstract}


\section{Introduction}

% The cognitive evolution of Large Language Models (LLMs) has progressed from intuitive ``System 1" paradigms to analytical ``System 2" architectures \citep{li202512surveyreasoning}. While System 1 models prioritize rapid inference, System 2 models, such as OpenAI's o1 \citep{o1} and DeepSeek-R1 \citep{deepseekai2025deepseekr1incentivizingreasoningcapability}, achieve superior reasoning performance through extended Chain-of-Thought (CoT). However, this capability incurs significant computational redundancy, commonly referred to as the "overthinking" phenomenon, where models needlessly expend excessive tokens on trivial queries \citep{chen2025think23overthinkingo1like}. To mitigate this, the field has pursued Efficient Reasoning, aiming to reconcile the capability of System 2 with the efficiency of System 1. %\citep{sui2025stopoverthinkingsurveyefficient}. 

The cognitive evolution of Large Language Models (LLMs) has established a functional dichotomy: lightweight ``System 1" models excel at rapid, intuitive responses, while ``System 2" models such as OpenAI's o1 \citep{o1} and DeepSeek-R1 \citep{deepseekai2025deepseekr1incentivizingreasoningcapability} achieve superior reasoning through extended Chain-of-Thought (CoT) \citep{li202512surveyreasoning}. Ideally, a unified model would seamlessly adapt between these modes, responding swiftly to simple queries while engaging deep deliberation for complex problems \citep{qwen3technicalreport}. However, training such an adaptive model remains an open challenge due to the zero-sum nature of training data \citep{shukor2025scalinglawsoptimaldata}. The verbose, self-reflective style of System 2 reasoning conflicts with the concise responses of System 1, causing mutual interference during joint optimization. Consequently, state-of-the-art LLMs typically release separate Thinking and Instruct checkpoints rather than a single unified model \citep{Qwen3-VL}.

%This incompatibility forces practitioners to maintain separate checkpoints for different use cases.

% 

%Within this landscape, model collaboration has emerged as a promising paradigm. Unlike single-model optimization approaches that rely on computationally intensive reinforcement learning (RL) to internalize adaptive reasoning strategies, collaborative frameworks strategically combine slow-thinking models and efficient general models. This synergy enables fine-grained, cost-effective control over the reasoning process without the substantial training overhead. \citep{han-etal-2025-token-budget}

\begin{figure}[!t]
  \centerline{\includegraphics[scale=0.26]{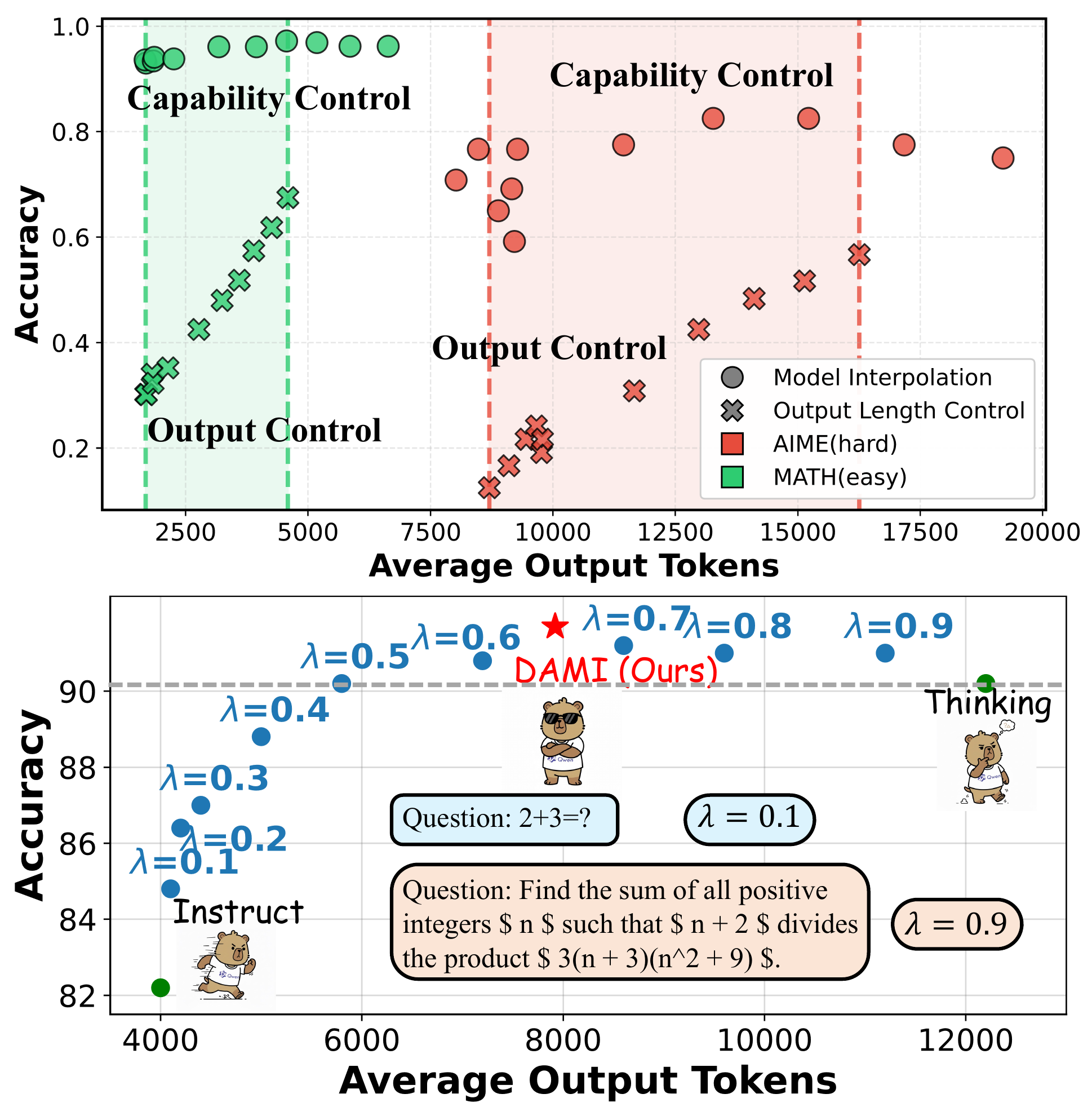}}
  \caption{Capability control outperforms output control across different efficiency levels on both tasks. Linear interpolation between Instruct and Thinking yields a monotonic Pareto frontier.}
  \label{fig-firstpage}
\end{figure}

Given this difficulty, the community has pursued a more tractable alternative: making System 2 models more efficient. This pursuit has spawned various approaches, from training-free methods like token budgeting \citep{han-etal-2025-token-budget}, early exit \citep{yang2025dynamicearlyexitreasoning}, and CoT compression \citep{xia2025tokenskipcontrollablechainofthoughtcompression}, to training-based methods employing reinforcement learning (RL) with length penalties \citep{arora2025traininglanguagemodelsreason,su2025thinkingfastrightbalancing} or trajectory pruning \citep{dai2025sgrpoearlyexitreinforcement,hou2025thinkprunepruninglongchainofthought}. Despite their diversity, these approaches share a common paradigm: \textbf{output control}, which limits \textit{what} models produce. We argue that this framing is misaligned: output length is merely a symptom of the model's cognitive configuration, not the root cause of inefficiency. As illustrated in Figure \ref{fig-firstpage}, constraining generation with token budgets risks truncating the reasoning process, resulting in incomplete chains of thought that severely degrade performance on challenging tasks.

In this work, we shift the focus from output control to \textbf{capability control}, which modulates \textit{how models think} rather than \textit{what they produce}. Our key observation is that by interpolating parameters between existing System 1 and System 2 checkpoints, we can smoothly adjust the model's cognitive intensity without additional training. As shown in Figure \ref{fig-firstpage} (upper), under comparable efficiency constraints, capability control consistently outperforms output control by producing naturally concise yet complete solutions.

A natural question arises: does parameter interpolation yield predictable behavior? Our pilot experiments in Section \ref{3.1} provide affirmative evidence. We find that increasing the interpolation coefficient $\lambda$ (where $\lambda$=0 corresponds to Instruct and $\lambda$=1 to Thinking) produces a smooth, monotonic Pareto frontier between accuracy and efficiency, as shown in Figure \ref{fig-firstpage} (lower), with continuous transitions in representation space. This predictability is grounded in the structural similarity between models: cosine similarity exceeds 0.99 across all layers, indicating that both models satisfy Linear Mode Connectivity \citep{frankle2020linearmodeconnectivitylotteryLMC} and reside within the same optimization basin. These findings establish parameter interpolation as a principled mechanism for capability control.

% Building on this foundation, we propose \textbf{DAMI} (\textbf{Qu}ery-\textbf{a}daptive \textbf{d}ynamic \textbf{M}odel \textbf{I}nterpolation), a framework that dynamically estimates a query-specific Reasoning Intensity $\lambda(q) \in [0,1]$ to configure the model's cognitive depth for each input. The core challenge lies in accurately predicting $\lambda$ under diverse deployment constraints. We explore two complementary settings: for training-based estimation, we develop a preference learning method that constructs pairwise comparisons jointly encoding accuracy and efficiency, achieving strong robustness with limited data; for training-free estimation, we introduce a confidence-based method that derives $\lambda$ from the cognitive discrepancy between models, offering superior generalization without additional training. Experiments on five mathematical reasoning benchmarks demonstrate that DAMI outperforms existing baselines in both accuracy and efficiency.

Building on this foundation, we propose \textbf{DAMI} (\textbf{D}yn\textbf{A}mic \textbf{M}odel \textbf{I}nterpolation), a framework that dynamically estimates a query-specific Reasoning Intensity $\lambda(q) \in [0,1]$ to configure the model's cognitive depth for each input. We explore two complementary methods tailored for different deployment scenarios: DAMI-Pref, a preference learning method that jointly encodes accuracy and efficiency criteria, is suited for accuracy-critical applications where in-domain training data is available; DAMI-Conf, a confidence-based method that leverages inter-model cognitive discrepancy, enables rapid deployment across diverse domains without requiring any training data. Experiments on five mathematical reasoning benchmarks show that DAMI improves accuracy by 1.6--3.4\% while reducing token consumption by 29--40\% compared to the Thinking model, outperforming static merging, early-exit, and routing baselines.

Our contributions are summarized as follows:
 \begin{itemize}
	\item We advocate a paradigm shift from output control to capability control for efficient reasoning, and identify dynamic model interpolation as a low-cost realization that leverages existing checkpoints, bypassing the difficulty of training  System1\&2 unified models
	\item We propose DAMI with two complementary $\lambda(q)$ 
    estimation methods: preference-based learning for data-available scenarios and confidence-based inference for zero-shot deployment.
    \item Extensive experiments demonstrate that DAMI consistently establishes superior Pareto frontiers across diverse benchmarks 
    and model families.
\end{itemize}

\section{Related Work}

\subsection{Efficient Reasoning}
The inherent trade-off between System 1's efficiency and System 2's accuracy has motivated a growing body of research on efficient reasoning, which can be categorized into single-model optimization and multi-model collaboration. Single-model approaches 
include early-exit methods that terminate generation upon sufficient 
confidence \citep{yang2025dynamicearlyexitreasoning, 
dai2025sgrpoearlyexitreinforcement}, CoT compression that removes 
redundant reasoning steps \citep{xia2025tokenskipcontrollablechainofthoughtcompression}, adaptive reasoning that adjusts output length based on query 
complexity \citep{luo2025adar1hybridcotbileveladaptive, 
shen2025dastdifficultyadaptiveslowthinkinglarge}, and other trajectory-revised methods \citep{yang2025testtimepromptintervention}. While adaptive methods share our motivation, they rely on RL with length penalties, incurring high training costs and potential performance degradation on difficult problems. Despite their diversity, these single-model approaches predominantly follow the output control paradigm, constraining \textit{what} models produce.
Multi-model approaches leverage complementary capabilities across 
models. Task decomposition methods delegate simple steps to lightweight 
models while reserving complex reasoning for stronger ones 
\citep{akhauri2025splitreasonlearningoffloadreasoning, 
fan2025pricesecondthoughtevaluation}. LLM Routing dynamically 
dispatches queries to suitable models from a pool 
\citep{ong2025routellmlearningroutellms, zhang2025routerr1teachingllmsmultiround}. 
Model merging statically integrates parameters of different models to combine 
their strengths \citep{wu2025unlockingefficientlongtoshortllm, Yao2025acm}. 
\textbf{Our work bridges routing and merging by enabling continuous, 
query-adaptive interpolation} rather than discrete model selection or static merging.

\subsection{Model Merging}
Model merging integrates parameters from multiple specialized models 
into a unified model without requiring original training data. It has 
been applied to mitigate catastrophic forgetting in continual learning 
\citep{schumann2023backwardcompatibilitydataupdates}, integrate 
knowledge in multi-task learning \citep{yang2024adamergingadaptivemodelmerging, 
ilharco2023editingmodelstaskarithmetic}, and aggregate updates in 
federated learning \citep{wang2020federatedlearningmatchedaveraging}. 
Recent work has introduced merging into efficient reasoning, with 
leading LLM teams releasing merged models that balance reasoning 
depth and efficiency \citep{Wu2025revisiting, kimiteam2025kimik15scalingreinforcement,DBLP:journals/corr/abs-2311-15316,yang-etal-2025-weights}. 
However, existing approaches apply fixed merging coefficients across 
all queries. To our knowledge, we are the first to propose query-aware dynamic merging between System 1 and System 2 models, enabling fine-grained integration of their complementary strengths for adaptive control over reasoning depth.

\section{Methodology}

\subsection{Why Interpolation Enable Capability Control} \label{3.1}

Before formulating our dynamic framework, we first validate that parameter interpolation can serve as a continuous and predictable control mechanism for reasoning. We conduct a pilot study using static linear interpolation between the Instruct model $\Theta^{(\text{I})}$ and the Thinking model $\Theta^{(\text{T})}$ with a fixed coefficient $\lambda \in [0,1]$ across the test set.
\begin{equation}
\Theta^{(\text{M})} = \lambda \Theta^{(\text{T})} + (1 - \lambda) \Theta^{(\text{I})}
\label{eq:dynamic_interpolation}
\end{equation}

\paragraph{Performance Monotonicity.} 
As shown in Figure \ref{fig-firstpage}, 
increasing $\lambda$ from 0 to 0.9 yields a smooth improvement in 
accuracy alongside a predictable increase 
in token consumption.  Crucially, this trajectory forms a convex Pareto frontier where intermediate interpolations can outperform both endpoints: achieving higher accuracy than the Instruct model while requiring fewer tokens than the Thinking model. 
This confirms that $\lambda$ serves as a reliable proxy for 
reasoning depth, enabling fine-grained efficiency-accuracy trade-offs.

\paragraph{Representation Continuity.} 
What underlies this behavioral smoothness? We analyze the 
hidden states of $\langle\text{think}\rangle$ token using Principal Component Analysis (PCA) across 50 samples on the MATH-500. As 
shown in Figure \ref{fig-pilot}(a), representations of 
interpolated models form a continuous trajectory connecting 
the Instruct and Thinking endpoints. The first principal component strongly correlates with 
$\lambda$ ($r$=0.974), indicating that interpolation induces a 
gradual cognitive transition rather than discrete jumps in the representation space.

\paragraph{Structural Connectivity.} 
This representational continuity is grounded in the geometric 
properties of the parameter space. Figure \ref{fig-pilot}(b) 
shows that the Instruct and Thinking models maintain high 
cosine similarity ($>$0.99) across all transformer layers, 
satisfying the Linear Mode Connectivity (LMC) property 
\citep{frankle2020linearmodeconnectivitylotteryLMC,zhan2025analyzingrolepermutationinvarianceLMC}. 
This indicates that both models reside within the same 
optimization basin, ensuring that interpolation traverses a 
valid path in the loss landscape. Notably, the L2 distance 
peaks in middle-to-deep layers, suggesting that System 2 
capabilities are encoded through targeted modifications in 
these regions rather than global parameter shifts.

\paragraph{Key Insight.} 

These findings establish a causal chain: structural alignment 
in parameter space $\rightarrow$ continuous transitions in 
representation space $\rightarrow$ monotonic Pareto frontier 
in performance. This connectivity transforms reasoning resource 
allocation into a well-posed prediction problem, where estimating 
the optimal $\lambda(q)$ directly determines model behavior. We provide formal theoretical justifications for each component of this causal chain, including proofs for interpolation 
validity (Appendix \ref{appendix:lmc}), performance monotonicity 
(Appendix \ref{appendix:monotonicity}), and representation continuity 
(Appendix \ref{appendix:lipschitz}). Together, these properties ensure that model interpolation behaves predictably and monotonically, providing the theoretical foundation for fine-grained, query-level dynamic merging. This validates our core premise: $\lambda$ can serve as a reliable control knob for reasoning intensity, enabling the adaptive framework we introduce next.

\begin{figure}[!t]
  \centerline{\includegraphics[scale=0.24]{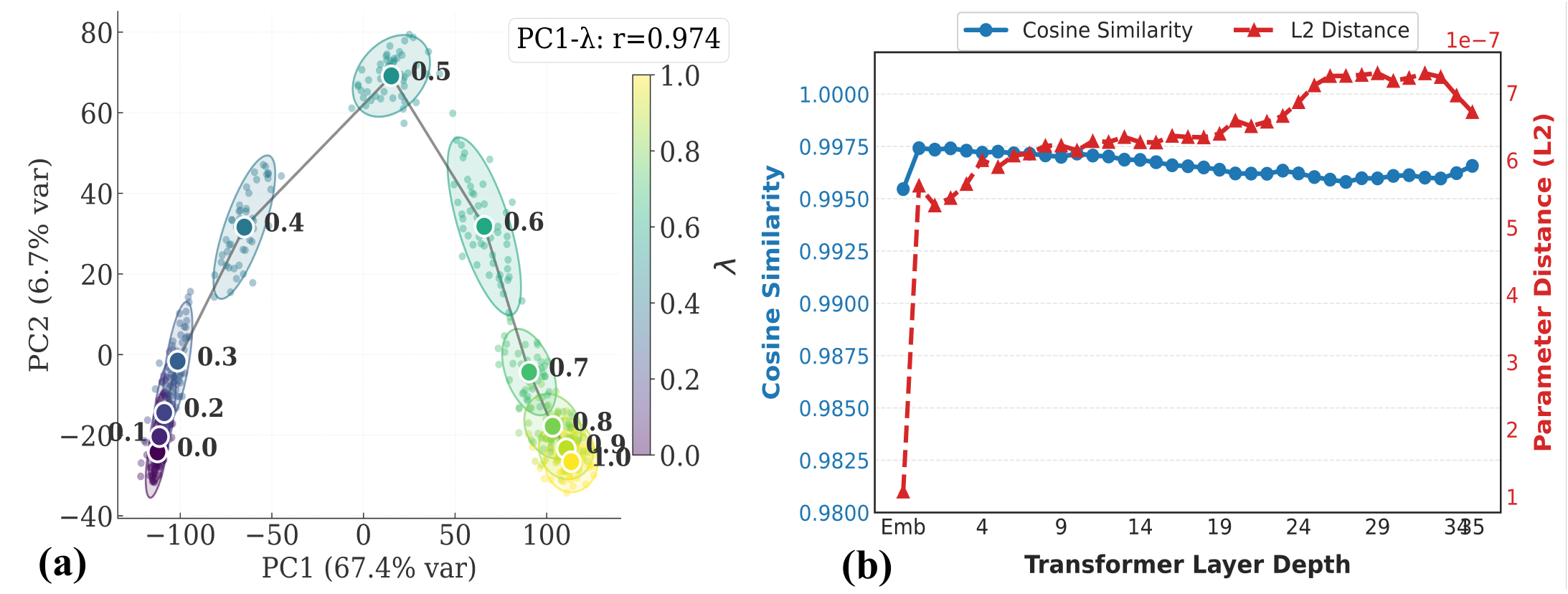}}
  \caption{The monotonicity and continuity of reasoning intensity.}
  \label{fig-pilot}
\end{figure} 

% \subsection{The Dynamic Interpolation Framework}

% The pilot study in Section 3.1 validates that a fixed $\lambda$ 
% enables smooth accuracy-efficiency trade-offs. However, 
% applying a uniform coefficient across all queries is inherently 
% suboptimal: simple queries are over-served while complex ones 
% may be under-resourced. This motivates our core extension—making 
% $\lambda$ a function of the input query $q$.

% In DAMI, we generalize Equation~\ref{eq:dynamic_interpolation} 
% by replacing the static coefficient with a query-adaptive 
% \textbf{Reasoning Intensity} $\lambda(q) \in [0, 1]$:
% \begin{equation}
% \Theta^{(\text{M})}(q) = \lambda(q) \Theta^{(\text{T})} + (1 - \lambda(q)) \Theta^{(\text{I})}
% \end{equation}
% This formulation preserves the monotonicity guarantees established 
% in Section 3.1 while enabling fine-grained, per-query resource 
% allocation. The technical challenge thus reduces to: 
% \textit{how to estimate the optimal $\lambda(q)$ for each query?}

% We explore this from two complementary perspectives based on 
% deployment constraints. When labeled data is available, we 
% investigate \textbf{training-based estimation} that learns to 
% predict $\lambda(q)$ from supervision (Section 3.3). When no 
% training data is accessible, we propose \textbf{training-free 
% estimation} that derives $\lambda(q)$ directly from model 
% confidence signals at inference time (Section 3.4).

\begin{figure*}[htbp]
  \centerline{\includegraphics[scale=0.35]{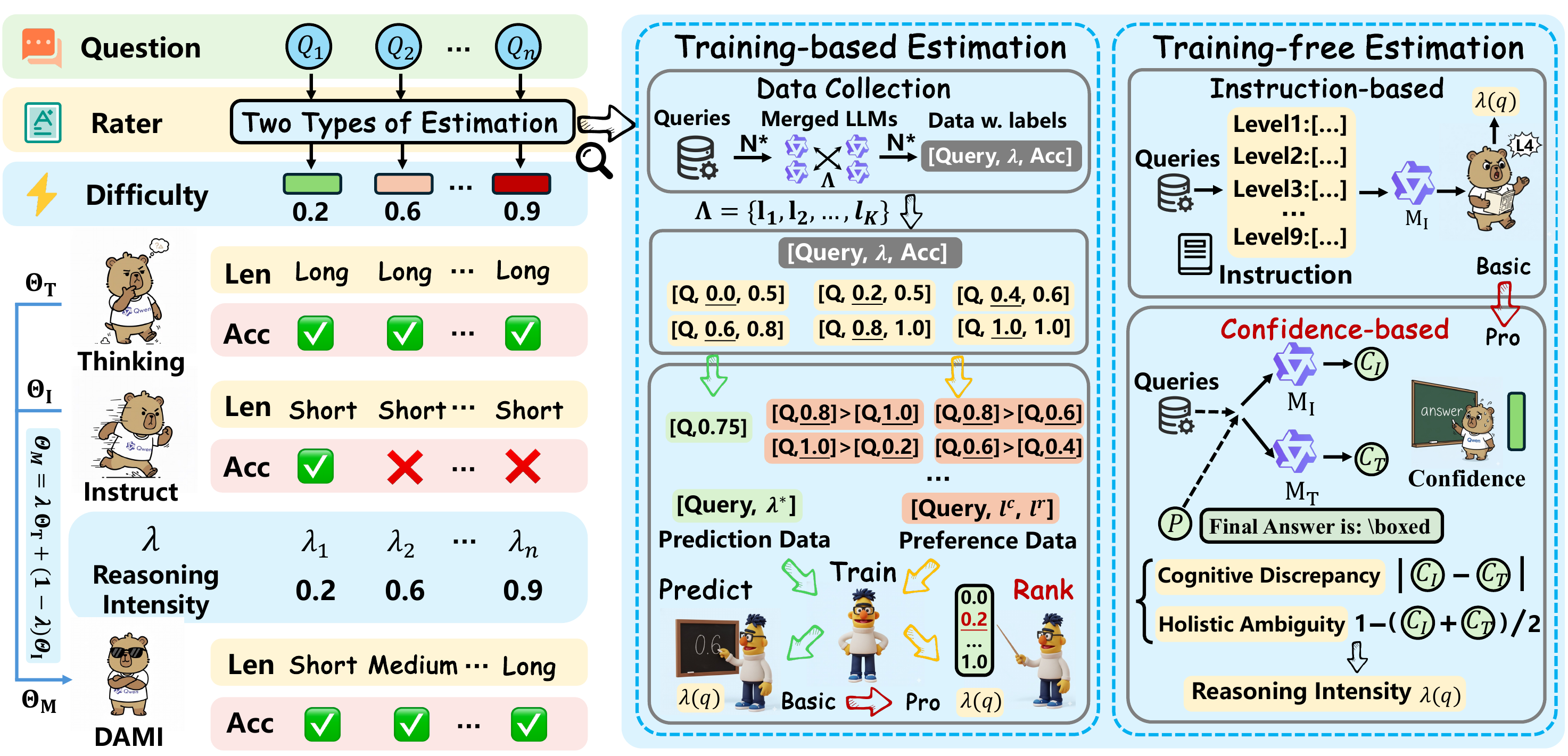}}
  \caption{An overview of our DAMI method. DAMI estimates query difficulty through either preference learning (training-based) or confidence signals (training-free), then dynamically adjusts the Reasoning Intensity $\lambda(q)$ to configure cognitive depth accordingly. }
  \label{fighuman}
\end{figure*}

\subsection{The Dynamic Interpolation Framework} \label{3.2}

While Section \ref{3.1} validates that parameter interpolation 
provides a reliable control mechanism, a fixed $\lambda$ 
is inherently suboptimal since different queries require different reasoning depths.  DAMI addresses this by making $\lambda$ query-dependent:
\begin{equation}
\Theta^{(\text{M})}(q) = \lambda(q) \Theta^{(\text{T})} + (1 - \lambda(q)) \Theta^{(\text{I})}
\end{equation}
We term $\lambda(q) \in [0, 1]$ the \textbf{Reasoning Intensity}, 
which modulates the cognitive depth for each input. The key 
technical challenge is estimating the optimal $\lambda(q)$. 
We explore two complementary methods: \textbf{DAMI-Pref}, a preference learning approach that jointly encodes accuracy and efficiency criteria, suited for accuracy-critical applications where in-domain training data is available; and \textbf{DAMI-Conf}, a confidence-based approach that leverages inter-model cognitive discrepancy, enabling rapid deployment across diverse domains without requiring any training data.

\subsection{Training-based Estimation of $\lambda$}

% When labeled data is available, a natural approach is to train 
% a model that predicts $\lambda(q)$ for each query. We first 
% describe a regression-based baseline, then present our improved 
% preference learning method.

\subsubsection{Baseline: Regression-based Estimation}

A straightforward solution is to train a lightweight router 
model $R_\phi$ to regress a pre-computed target $\lambda^*(q)$. 
We profile each query $q$ across $J$ discrete coefficients 
$\Lambda = \{l_1, \dots, l_J\}$, measuring accuracy via 
$N$-sample evaluation:
\begin{equation}
\text{Acc}^{l_j}(q) = \frac{1}{N} \sum_{n=1}^{N} \mathbb{I}(\mathcal{M}^{l_j}_n(q), a)
\end{equation}
where $\mathcal{M}^{l_j}_n(q)$ is the response from a model 
merged with coefficient $l_j$, and $\mathbb{I}(\cdot, \cdot)$ 
is a binary evaluation function. The target $\lambda^*(q)$ is 
defined as the minimum coefficient achieving the highest accuracy, 
and the router is trained with MSE loss:
\begin{equation}
\mathcal{L}(\phi) = \frac{1}{|\mathcal{D}_{\text{train}}|} 
\sum_{q \in \mathcal{D}_{\text{train}}} (R_{\phi}(q) - \lambda^*(q))^2
\end{equation}

However, this approach suffers from two weaknesses: (1) point-wise 
labels $\lambda^*(q)$ are inherently noisy due to sampling variance, 
and (2) regression requires large-scale data for robust generalization. 
These limitations motivate our preference-based method.

\subsubsection{DAMI-Pref: Preference-based Estimation}

Instead of regressing noisy labels, we recast $\lambda$ estimation 
as a preference learning problem. A lightweight reward model $R_\psi$ 
is trained to score the relative efficacy of different coefficients, 
offering three advantages: (1) pairwise comparisons are more robust 
to noise, (2) preference criteria can jointly encode accuracy and 
efficiency, and (3) the method requires substantially less training data.

\paragraph{Preference Data Construction.}
We construct policy performance tuples $P(q, l_j) = \langle l_j, 
\text{Acc}(q, l_j), \text{Cost}(q, l_j) \rangle$, where Cost is 
the average token count. Preference pairs are generated according 
to a hierarchical criterion:
\begin{equation}
P_i \succ P_j \iff
\begin{cases}
\text{Acc}_i > \text{Acc}_j + \delta_{\text{acc}}  \\
|\text{Acc}_i - \text{Acc}_j| \le \delta_{\text{acc}} \land \text{Cost}_i < \text{Cost}_j 
\end{cases}
\end{equation}
where $\delta_{\text{acc}}$ is a small tolerance threshold to account for sampling noise.
This criterion prioritizes accuracy while using efficiency as 
a tiebreaker, directly reflecting the practical goal of 
achieving high performance at minimal cost. Pairs without 
clear preference are discarded to ensure signal quality. This process is applied to all pairs to construct the final dataset:
\begin{equation}
\mathcal{D}_{\text{pref}} = \{ (q, l_c, l_r) \mid P(q, l_c) \succ P(q, l_r) \}
\end{equation}

\paragraph{Optimization and Inference.}
The reward model is optimized via binary cross-entropy loss:
\begin{equation}
\mathcal{L}(\psi) = - \mathbb{E}_{(q, l_c, l_r) \sim \mathcal{D}_{\text{pref}}} 
\left[ \log \sigma \left( R_{\psi}(q, l_c) - R_{\psi}(q, l_r) \right) \right]
\end{equation}
% At inference, we select the coefficient maximizing the 
% predicted reward over a candidate set $\Lambda_{\text{cand}}$:
% \begin{equation}
% \lambda(q) = \underset{l \in \Lambda_{\text{cand}}}{\arg\max} \, R_{\psi}(q, l)
% \end{equation}

\newcommand{\annotate}[3]{%
    #1\raisebox{-0.5ex}{\scriptsize\textcolor{#2}{#3}}%
}

\setlength{\tabcolsep}{2.5pt}
\begin{table*}[h!]
\centering
\scalebox{0.75}{
\begin{tabular}{@{}l|cccc|cccc|cccc|cccc|cccc|ccc@{}} 
\toprule
 \multirow{3}{*}{\textbf{Method}} 
 & \multicolumn{4}{c|}{\textbf{GSM8K}}& \multicolumn{4}{c|}{\textbf{MATH-500}} & \multicolumn{4}{c|}{\textbf{AMC}} & \multicolumn{4}{c|}{\textbf{AIME24}}  & \multicolumn{4}{c|}{\textbf{AIME25}}  & \multicolumn{3}{c}{\textbf{Overall}}
 \\
   & {Acc$\uparrow$} & {Tok$\downarrow$} & CR$\downarrow$ & {Think} 
   & {Acc$\uparrow$} & {Tok$\downarrow$} & CR$\downarrow$ & {Think} 
   & {Acc$\uparrow$} & {Tok$\downarrow$} & CR$\downarrow$ & {Think} 
   & {Acc$\uparrow$} & {Tok$\downarrow$} & CR$\downarrow$ & {Think} 
   & {Acc$\uparrow$} & {Tok$\downarrow$} & CR$\downarrow$ & {Think} 
   & {Acc$\uparrow$} & CR$\downarrow$ & {Think}   \\ 

   & {2$\times$} & \#N & \% & \#R 
   & {2$\times$} & \#N & \% & \#R 
   & {8$\times$} & \#N & \% & \#R 
   & {8$\times$} & \#N & \% & \#R 
   & {8$\times$} & \#N & \% & \#R 
   & - & \% & \#R  \\ 
   
 \midrule

\multicolumn{24}{l}{{\cellcolor[rgb]{0.957,0.957,0.957}}\textit{\textbf{Qwen3-4B-2507-Thinking/Qwen3-4B-2507-Instruct}}} \\
Instruct & 94.3 & 349 & 23.1 & 0 & 93.1 & 1741 & 26.1 & 0 & 93.8 & 3034 & 27.4 & 0 & 61.7 & 8938 & 46.5 & 0 & 49.2 & 8384 & 39.3 & 0 & 78.4 & 37.5 & 0 \\
\rowcolor[rgb]{0.957,0.878,0.702}
Thinking & \textbf{95.5} & 1510 & 100 & 100 & 96.1 & 6668 & 100 & 100 & 99.4 & 11076 & 100 & 100 & 78.3 & 19241 & 100 & 100 & 71.7 & 21340 & 100 & 100 & 88.2 & 100 & 100 \\
\hline
DEER & \textbf{95.5} & 1059 & 70 & 100 & 95.4 & 6125 & 92 & 100 & \textbf{100} & 10084 & 91 & 100 & 83.3 & 18065 & 94 & 100 & 76.7 & 19287 & 90 & 100 & 90.2 & 87 & 100 \\
\hline
TA & 95 & 1225 & 81 & 100 & \textbf{96.5} & 4671 & 70 & 100 & \textbf{100} & 7760 & 70 & 100 & 78.3 & 13515 & 70 & 98 & 73.3 & 15546 & 73 & 98 & 88.6 & 73 & 99 \\
TIES & 94 & 363 & 24 & 0 & 94.8 & 1899 & 28 & 0 & 95.6 & 3012 & 27 & 0 & 68.3 & 8141 & 42 & 0 & 51.7 & 8419 & 39 & 0 & 80.9 & 32 & 0 \\
MI-03 & 94 & 383 & 25 & 0 & 95.3 & 1933 & 29 & 0 & 96.2 & 3381 & 31 & 0 & 67.5 & 9708 & 50 & 0 & 51.7 & 9280 & 43 & 0 & 81 & 36 & 0 \\
MI-05 & 95.1 & 1084 & 72 & 83 & 96.1 & 3213 & 48 & 72 & 96.2 & 5138 & 46 & 61 & 73.3 & 9321 & 48 & 50 & 65.8 & 10533 & 49 & 54 & 85.3 & 53 & 64 \\
MI-07 & \textbf{95.4} & 1222 & 81 & 100 & \textbf{96.4} & 4617 & 69 & 100 & \textbf{100} & 7501 & 68 & 100 & \textbf{82.5} & 13993 & 73 & 100 & 75 & 16086 & 75 & 98 & 89.9 & 73 & 100 \\
\hline
Routing & 94.5 & 430 & 28 & 2.7 & \textbf{96.4} & 5435 & 82 & 53.8 & \textbf{100} & 10356 & 93 & 85 & 76.7 & 19861 & 103 & 100 & 73.3 & 21035 & 99 & 100 & 88.2 & 81 & 68.3 \\
\textit{DAMI-Pred} & 94.4 & 365 & 24 & 0 & 95.6 & 1886 & 28 & 0 & 95 & 3477 & 31 & 0 & 73.3 & 7703 & 40 & 0 & 53.3 & 8776 & 41 & 0 & 83.3 & 33 & 0 \\
\textit{DAMI-Prompt} & \textbf{95.5} & 1028 & 68 & 59 & 95 & 2832 & 42 & 27 & 95.6 & 4715 & 43 & 30 & 80 & 10648 & 55 & 76 & 75 & 13009 & 61 & 88 & 88.2 & 54 & 56 \\
\rowcolor[rgb]{0.87,0.94,1}
\textit{DAMI-Pref} & 94.5 & 843 & 56 & 25 & \textbf{96.4} & 4315 & 65 & 64 & \textbf{98.8} & 7872 & 71 & 86 & \textbf{85} & 15412 & 80 & 95 & \textbf{83.3} & 17488 & 82 & 92 & \textbf{91.6} & 71 & 72 \\
\rowcolor[rgb]{0.87,0.94,1}
\textit{DAMI-Conf} & 94.5 & 856 & 57 & 41 & 96.2 & 2954 & 44 & 34 & \textbf{98.8} & 5891 & 53 & 55 & \textbf{85} & 13485 & 70 & 79 & \textbf{79.2} & 16412 & 77 & 94 & \textbf{90.7} & 60 & 61 \\

 \midrule
\multicolumn{24}{l}{{\cellcolor[rgb]{0.957,0.957,0.957}}\textit{\textbf{DeepSeek-R1-Distill-Qwen-7B/Qwen2.5-Math-7B}}} \\
Instruct & 55.3 & 1007 &- & 0 & 52 & 1268 &- & 0 & 42.5 & 1463 &- & 0 & 20 & 1470 &- & 0 & 3.3 & 1104 &- & 0 & 34.6 & - & 0 \\
\rowcolor[rgb]{0.957,0.878,0.702}
Thinking & 89.6 & 1886 & 100 & 100 & 86 & 5881 & 100 & 100 & \textbf{87.5} & 8600 & 100 & 100 & 43.3 & 17887 & 100 & 100 & \textbf{30} & 21681 & 100 & 100 & 67.3 & 100 & 100 \\
\hline
DEER & 89.5 & 873 & 46 & 100 & 88.8 & 3140 & 53 & 100 & 82.5 & 7398 & 86 & 100 & \textbf{45} & 13945 & 78 & 100 & \textbf{30} & 15941 & 74 & 100 & 67.2 & 67 & 100 \\
\hline
TA & 90.7 & 769 & 41 & 99 & 88.8 & 2670 & 45 & 97 & 78.1 & 4319 & 50 & 96 & 39.2 & 9396 & 53 & 92 & 25.8 & 11189 & 52 & 83 & 64.5 & 48 & 93 \\
TIES & 86.5 & 346 & 18 & 0 & 80.2 & 827 & 14 & 0 & 58.1 & 1263 & 15 & 0 & 16.7 & 1739 & 10 & 0 & 10 & 1827 & 8 & 0 & 50.3 & 13 & 0 \\
MI-03 & 85.5 & 379 & 20 & 0 & 76.8 & 1501 & 26 & 0 & 57.5 & 1817 & 21 & 0 & 23.3 & 3459 & 19 & 0 & 16.7 & 4485 & 21 & 0 & 52 & 21 & 0 \\
MI-05 & \textbf{91.9} & 617 & 33 & 71 & 83.6 & 2338 & 40 & 13 & 72.5 & 4462 & 52 & 8 & 20 & 15330 & 86 & 0 & 20 & 10665 & 49 & 0 & 57.6 & 52 & 18 \\
MI-07 & 91.2 & 902 & 48 & 99 & 88.2 & 3243 & 55 & 95 & 85 & 4867 & 57 & 92 & 36.7 & 18693 & 105 & 50 & 30 & 17169 & 79 & 60 & 66.2 & 69 & 79 \\
\hline
Routing & 64.5 & 1045 & 55 & 0 & 63.6 & 1303 & 22 & 4.4 & 50 & 1425 & 17 & 5 & 25 & 1875 & 10 & 20 & 6.7 & 2851 & 13 & 10 & 42 & 24 & 7.9 \\
\textit{DAMI-Pred} & 83.5 & 410 & 22 & 0 & 79.6 & 1823 & 31 & 0 & 62.5 & 3103 & 36 & 20 & 26.7 & 13209 & 74 & 31 & 13.3 & 13813 & 64 & 26 & 53.1 & 45 & 15.4 \\
\textit{DAMI-Prompt} & 82.1 & 431 & 23 & 0 & 76.2 & 1657 & 28 & 0 & 55 & 1697 & 20 & 0 & 20 & 3962 & 22 & 0 & 10 & 4962 & 23 & 0 & 48.7 & 23 & 0 \\
\hline
\rowcolor[rgb]{0.87,0.94,1}
\textit{DAMI-Pref} & \textbf{92.4} & 1121 & 59 & 48 & \textbf{89.4} & 3952 & 67 & 75 & \textbf{88.8} & 5819 & 68 & 88 & \textbf{45} & 14561 & 81 & 92 & \textbf{32.5} & 16837 & 78 & 94 & \textbf{69.6} & 71 & 79.4 \\
\rowcolor[rgb]{0.87,0.94,1}
\textit{DAMI-Conf} & 91.8 & 1050 & 56 & 55 & \textbf{89} & 3526 & 60 & 62 & \textbf{87.5} & 5299 & 62 & 78 & \textbf{44.2} & 13835 & 77 & 85 & \textbf{30} & 16174 & 75 & 90 & \textbf{68.9} & 66 & 74 \\

 \bottomrule
\end{tabular}
}
\caption{Experimental results on two model pairs (Qwen3-4B and Qwen2.5-7B) across five mathematical reasoning benchmarks. "Acc" denotes accuracy, "Tok" denotes token count, "CR" denotes compression rate, and "Think" denotes thinking ratio. $\uparrow$/$\downarrow$ indicate that higher/lower values are better. The top-2 best results are highlighted in \textbf{bold}. The result is statistically significant with $p$-value $<$ 0.05.}
\label{baselines}
\end{table*}

\subsection{Training-free Estimation of $\lambda$}

% Training-based methods require labeled data and incur 
% annotation costs. For zero-shot deployment, we explore 
% approaches that derive $\lambda(q)$ directly at inference 
% time by leveraging the models' intrinsic capabilities. 

%We first describe a prompt-based baseline, then present 
% our confidence-based method that offers superior robustness 
% and generalization.

\subsubsection{Baseline: Prompt-based Estimation}

A straightforward approach is to leverage the model's in-context 
learning for difficulty assessment. We design a zero-shot prompt 
that instructs the Instruct model to rate query difficulty on a 
scale from 1 to 9 (see Figure \ref{fig:rating_prompt} in Appendix \ref{DAMI-Prompt}), then linearly scale to obtain 
$\lambda(q)$. However, this method depends heavily on instruction-following 
capability and generalizes poorly across model families.

\subsubsection{DAMI-Conf: Confidence-based Estimation}

Considering the above drawbacks, we propose to derive $\lambda(q)$ from intrinsic confidence signals, 
offering two advantages: model-agnostic signals that generalize 
across architectures, and continuous-valued output for fine-grained control.

\paragraph{Confidence Extraction.}
Following \cite{yang2025dynamicearlyexitreasoning}, we define 
confidence as the geometric mean of maximum token probabilities 
over the answer sequence $\bm{a} = (a_1, \dots, a_n)$:
\begin{equation}
\mathcal{C}(\mathcal{M}, q) = \left( \prod_{t=1}^{n} 
\max_{a_t \in \mathcal{V}} \text{softmax}(\mathcal{M}(q, i, \bm{a_{<t}}))_t \right)^{1/n}
\end{equation}
where $i$ is an answer-inducing prompt like "Final answer is: \textbackslash boxed\{". We compute $\mathcal{C}^I(q)$ 
and $\mathcal{C}^T(q)$ for the Instruct and Thinking models respectively.

\paragraph{Dual-Signal Fusion.}
We construct $\lambda(q)$ from two complementary signals that 
capture different aspects of query difficulty.
\textit{Holistic Ambiguity} measures overall system uncertainty. 
When neither model is confident, the query is likely challenging 
and requires deeper reasoning:
\begin{equation}
S_{\text{ambi}}(q) = 1 - \frac{\mathcal{C}^I(q) + \mathcal{C}^T(q)}{2}
\end{equation}
\textit{Cognitive Discrepancy} quantifies the disagreement between 
models by measuring the difference in their confidence. 
A high discrepancy suggests that the two models employ different 
strategies for addressing the query, indicating potential ambiguity 
that requires extensive reasoning to resolve:
\begin{equation}
S_{\text{dis}}(q) = |\mathcal{C}^I(q) - \mathcal{C}^T(q)|
\end{equation}

The fused signal $S_{\text{final}} = S_{\text{ambi}} + S_{\text{dis}}$ 
is transformed via calibrated sigmoid to yield Reasoning Intensity:
\begin{equation}
\lambda(q) = \sigma \left( \frac{S_{\text{final}}(q) - \mu}{\tau} \right)
\end{equation}
where $\mu$ and $\tau$ control the decision boundary and sensitivity.

\section{Experiments}

\subsection{Experimental Setup}

\textbf{Models and Datasets.}
We conduct experiments on the following two pairs of open-source models: (Qwen3-4B-2507-Thinking, Qwen3-4B-2507-Instruct) and (DeepSeek-R1-Distill-Qwen-7B, Qwen2.5-Math-7B). The evaluation benchmarks primarily consist of mathematical reasoning tasks, encompassing five specific datasets ranging from simple to difficult: GSM8K \citep{cobbe2021trainingverifierssolvemathgsm8k}, MATH-500 \citep{math500hendrycks2021measuringmathematicalproblemsolving}, AMC 2023 \citep{AMC2023}, AIME 2024, and AIME 2025 \citep{aime}. For the training of the router model and reward model, we select training data from DeepMath \citep{he2025deepmath103klargescalechallengingdecontaminated}, which is out-of-distribution (OOD) relative to the evaluation sets. Since this dataset inherently contains difficulty labels, we first conduct an initial screening based on these labels, selecting an average of 1,000 samples across various difficulty levels for constructing the training data described in Section \ref{3.2}.

% 我们在以下两对开源模型上进行实验：(Qwen3-4B-2507-Thinking, Qwen3-4B-2507-Instruct) and (DeepSeek-R1-Distill-Qwen-7B, Qwen2.5-Math-7B)。评测的benchmark主要为数学推理类任务，从简单到困难有以下5个具体的数据集：GSM8K \citep{cobbe2021trainingverifierssolvemathgsm8k}, MATH-500 \citep{math500hendrycks2021measuringmathematicalproblemsolving}, AMC 2023 \citep{AMC2023}, AIME 2024, AIME 2025 \citep{aime}.

% 对于router model和reward model的训练，我们选用了与评测集ood的训练数据 DeepMath \citep{he2025deepmath103klargescalechallengingdecontaminated}。由于该数据集本身含有难度标签，所以我们先基于该标签对数据进行初筛，平均选取各难度样本共计5000条，用于3.2中训练数据的构造。

\noindent \textbf{Baselines.} We compare our methods (DAMI-Pref, DAMI-Conf) against three categories of approaches: (1) \textbf{output control methods}, including the early-exit approach DEER \citep{yang2025dynamicearlyexitreasoning}; (2) \textbf{static capability control methods}, including Task Arithmetic (TA) \citep{ilharco2023editingmodelstaskarithmetic}, TIES-Merging (TIES) \citep{yadav2023tiesmergingresolvinginterferencemerging}, and Model Interpolation (MI) \citep{Wu2025revisiting}, which apply fixed merging coefficients across all queries; and (3) \textbf{dynamic capability control methods}, including prompt-based LLM Routing, as well as two strong dynamic merging baselines introduced in this work (DAMI-Pred, DAMI-Prompt). We also include the two original models (Instruct, Thinking) as reference endpoints. The hyperparameters for all baseline methods are set strictly according to the original papers.
%More implementation details like hyperparameters are provided in the Appendix \ref{baseline_settings}.

\begin{figure}[!t]
  \centerline{\includegraphics[scale=0.36]{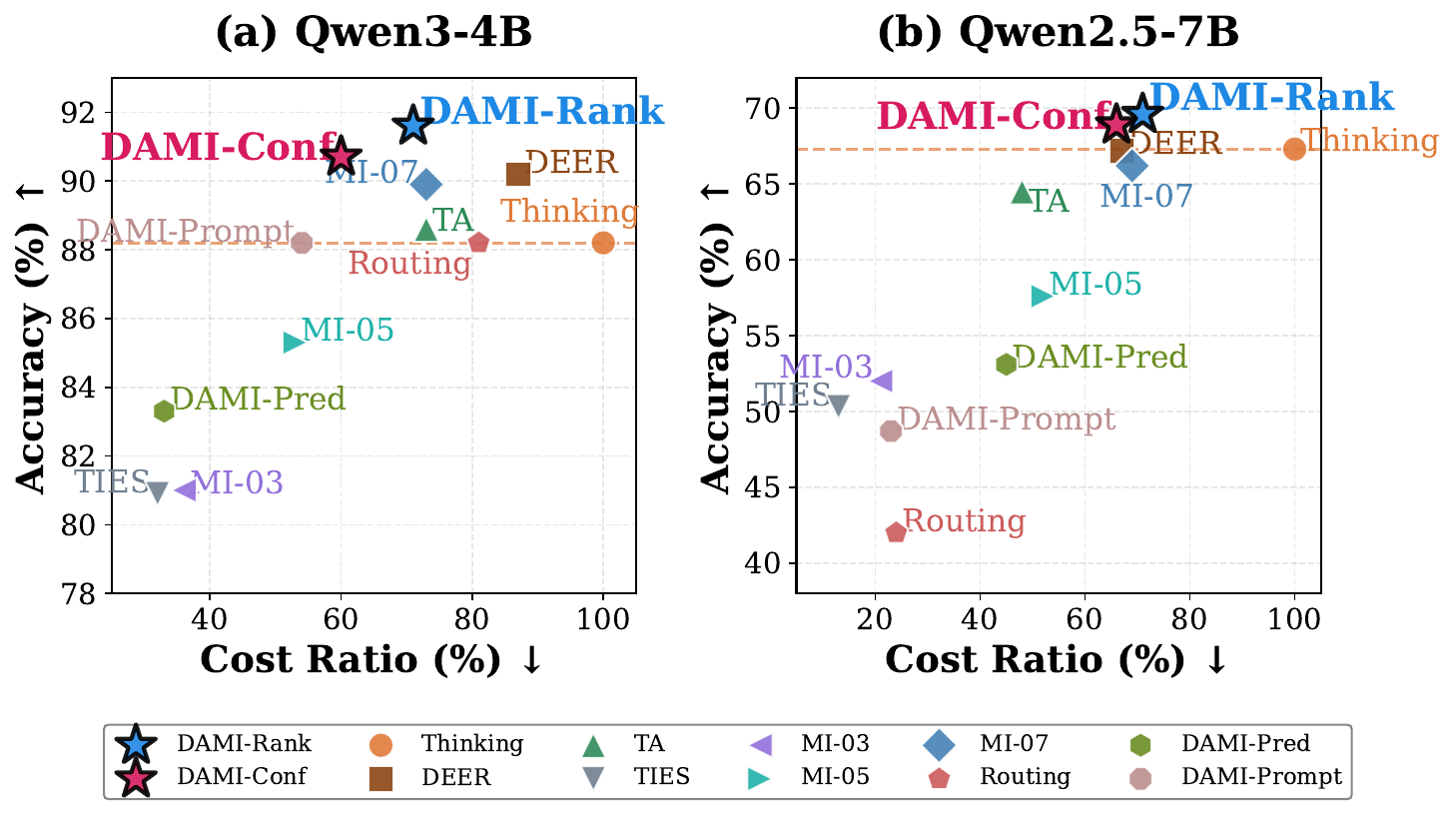}}
  \caption{Accuracy-efficiency scatter plot aggregated over five benchmarks. DAMI methods (starred) achieve the best accuracy-efficiency trade-offs on both Qwen3-4B and Qwen2.5-7B.}
  \label{fig:main_scatter}
\end{figure}  
%We select three of the most popular model merging methods for evaluation, including Task Arithmetic (TA) \citep{ilharco2023editingmodelstaskarithmetic}, TIES-Merging (TIES) \citep{yadav2023tiesmergingresolvinginterferencemerging}, and Model Interpolation (MI) \citep{Wu2025revisiting}, along with the two initial models used for merging, and compare them against our Dynamic Merging (DAMI) methods.

% 我们选取了最受欢迎的三个模型融合方法, including Task Arithmetic (TA), TIES-Merging (TIES),  and Model Interpolation (MI)，基于早退的高效推理方法(DEER)，基于prompt判断问题难度的LLM Routing方法 (Routing)，两个我们在方法章节选为baseline的方法(DAMI-Pred,DAMI-Prompt)以及用于融合的两个初始模型原始推理结果与我们的两种种Dynamic Merging(DAMI-Pref, DAMI-Conf)方法进行比较，更多关于baseline设置的细节见附录

\begin{figure}[!t]
  \centerline{\includegraphics[scale=0.07]{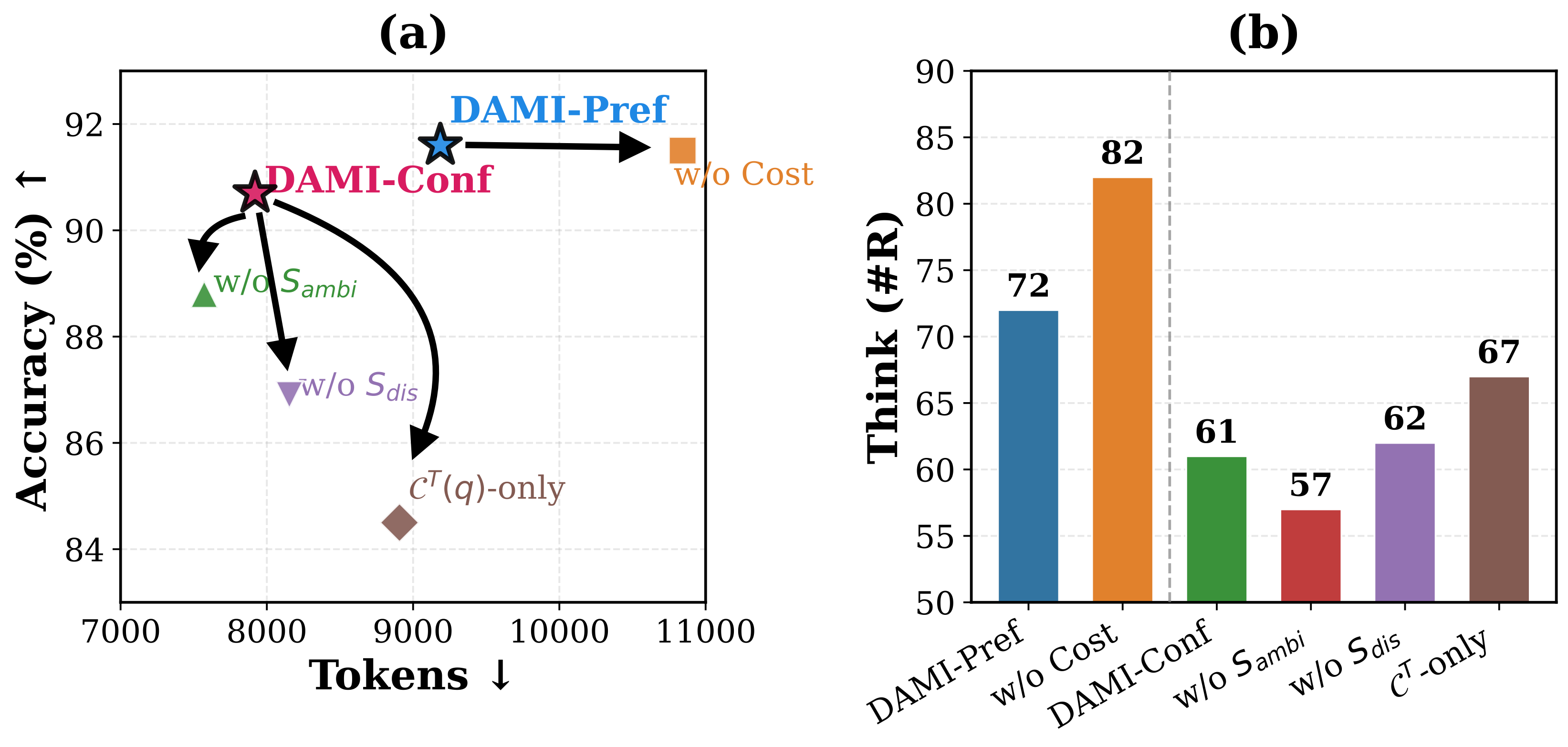}}
  \caption{Ablation study results on the Qwen3-4B model pair.}
  \label{fig:ablation}
\end{figure}

\noindent \textbf{Metrics and Implementations.} We selected \textit{Accuracy} (\textbf{Acc n$\times$}), \textit{Token Number} (\textbf{Tok \#N}), and \textit{Thinking Ratio} (\textbf{Think \#R}) as our evaluation metrics. \textbf{Acc n$\times$} denotes the final answer accuracy, and n $\times$ represents the number of sampling rounds conducted for evaluation.
\textbf{Tok \#N} indicates the average number of generated tokens per sample, serving as a measure of computational cost. \textbf{Think \#R} is defined as the percentage of responses containing the $\langle\text{/think}\rangle$ token, quantifying the prevalence of explicit CoT reasoning. Given the limited sample sizes in the AMC 2023, AIME 2024, and AIME 2025 datasets, we perform 8 sampling rounds per instance and average the results across all metrics to ensure statistical stability and reliability. Across all models, we consistently apply the hyperparameters $T=0.6$ and Top-p$=0.95$. For confidence-based estimation, we set $\mu=0.3$ and $\tau=0.3$.

\begin{figure}[!t]
  \centerline{\includegraphics[scale=0.5]{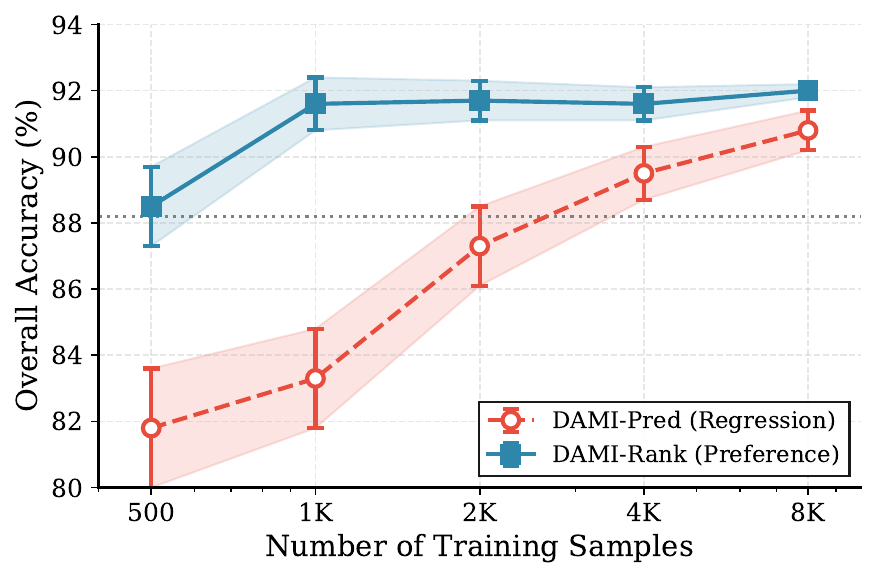}}
  \caption{Effect of training data size on training-based $\lambda$ estimation.}
  \label{fig:data_efficiency}
\end{figure}

\subsection{Experimental Results}

\noindent \textbf{Overall Performance.}
As shown in Table~\ref{baselines}, DAMI-Pref and DAMI-Conf 
consistently achieve the best performance across both model 
families. On Qwen3-4B pair, DAMI-Pref attains 91.6\% accuracy (+3.4\% absolute) while 
reducing token consumption by 29\% compared to vanilla 
Thinking. DAMI-Conf further improves efficiency with 40\% 
token reduction while maintaining 90.7\% accuracy (+2.5\% absolute). Figure~\ref{fig:main_scatter} visualizes the overall accuracy-efficiency trade-off: DAMI-Pref and DAMI-Conf consistently occupy the Pareto-optimal region, dominating both output control methods and static capability control methods across both model pairs.

\begin{figure}[!t]
  \centerline{\includegraphics[scale=0.24]{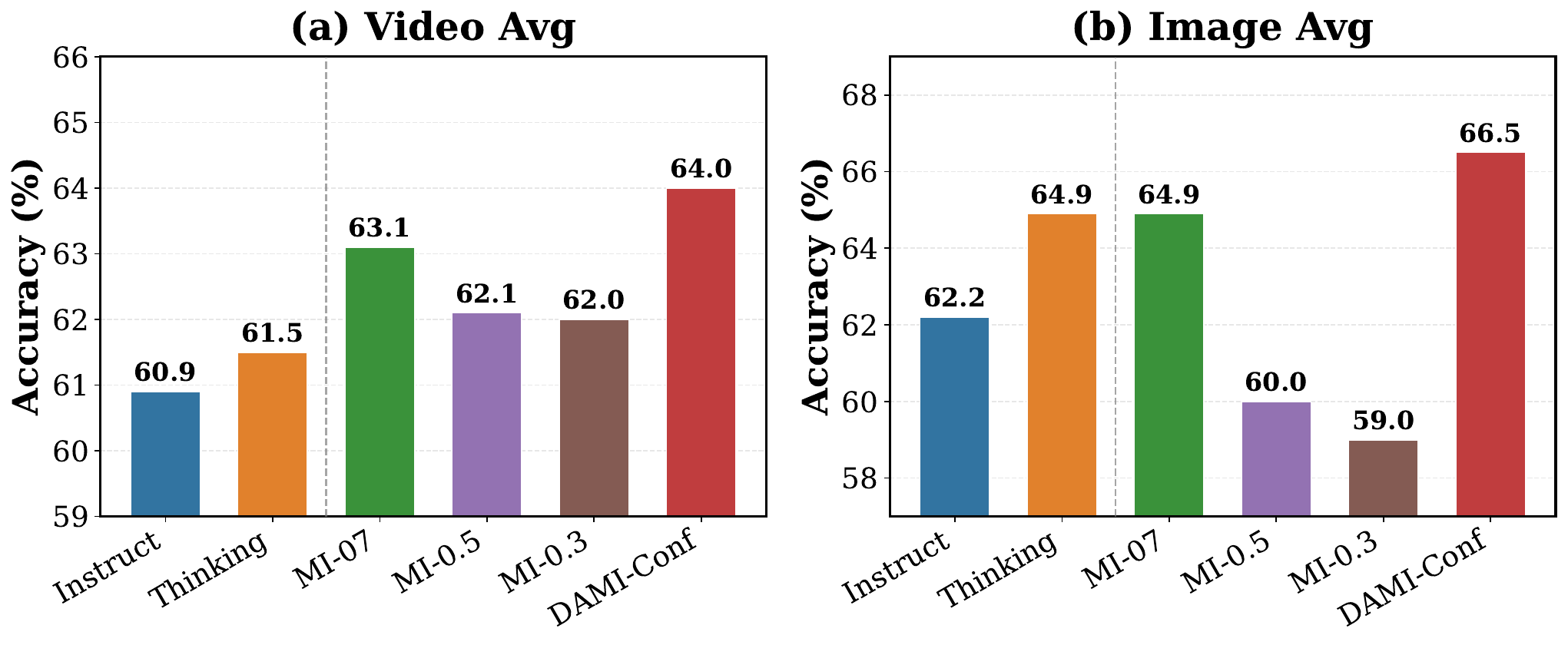}}
  \caption{Generalization to multimodal tasks of DAMI.}
  \label{fig:video_image_avg}
\end{figure}  

\noindent \textbf{Comparison of Strong Dynamic Merging Baselines.} %Estimation Methods.}
% The regression-based baseline (DAMI-Pred) exhibits unstable 
% predictions that consistently underestimate the required 
% reasoning intensity, resulting in degraded accuracy. On the 
% Qwen3-4B pair, DAMI-Pred achieves only 83.3\% accuracy compared 
% to 91.6\% for DAMI-Pref. The prompt-based baseline (DAMI-Prompt) 
% performs reasonably on Qwen3-4B but suffers severe 
% degradation on the 7B model pair, as Qwen2.5-Math-7B 
% lacks robust instruction-following capability for reliable 
% difficulty assessment. In contrast, our preference learning 
% and confidence-based methods demonstrate consistent effectiveness 
% across both model families. DAMI-Pref benefits from pairwise 
% comparisons that are inherently more robust to label noise, 
% while DAMI-Conf leverages model-agnostic confidence signals that 
% generalize across architectures.
The DAMI-Pred consistently underestimates 
required reasoning intensity, achieving only 83.3\% accuracy 
versus 91.6\% for DAMI-Pref on Qwen3-4B. The 
DAMI-Prompt performs reasonably on Qwen3-4B but degrades 
severely on 7B, as Qwen2.5-Math-7B lacks robust 
instruction-following for difficulty assessment. In contrast, 
DAMI-Pref benefits from pairwise comparisons that are robust 
to label noise, while DAMI-Conf leverages model-agnostic 
confidence signals that generalize across architectures.

% \begin{figure}[!t]
%   \centerline{\includegraphics[scale=0.5]{./pic/pareto_optimality_real.pdf}}
%   \caption{Performance of DAMI. The interpolation coefficient $\lambda$
%  serves as a controllable proxy for reasoning depth, exhibiting a strong monotonic relationship with token consumption and accuracy, enabling precise efficiency-performance trade-offs.}
%   \label{fig-example}
% \end{figure}  

\begin{figure*}[htbp]
  \centerline{\includegraphics[scale=0.4]{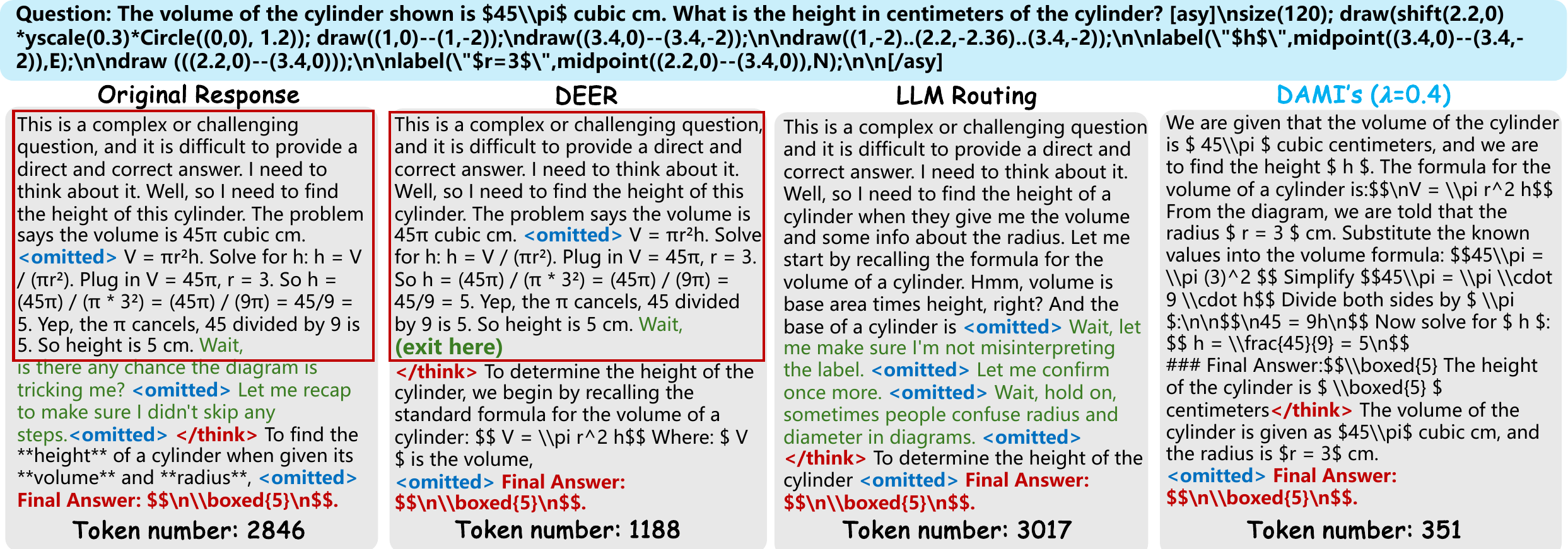}}
  \caption{Comparison of responses generated by four methods. DAMI produces 
 efficient reasoning by modulating cognitive configuration.}
  \label{fig:case}
\end{figure*} 

\begin{figure}[htbp]
  \centerline{\includegraphics[scale=0.43]{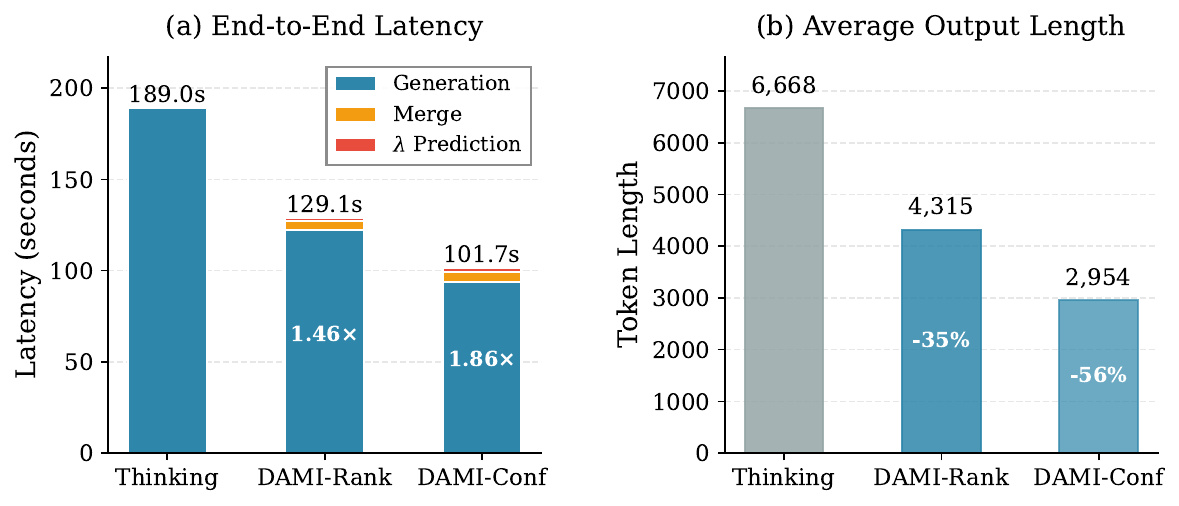}}
  \caption{Computational cost analysis on MATH-500 (Qwen3-4B). (a) End-to-end latency breakdown showing $\lambda$ prediction, parameter merging, and generation time. (b) Average output token length.}
  \label{fig:computational_cost}
\end{figure}

\noindent \textbf{Comparison with SoTA Methods.}
DEER achieves comparable accuracy to DAMI-Pref on Qwen3-4B but 
consumes more tokens and risks truncated reasoning. LLM 
Routing suffers from its discrete nature: on Qwen2.5-7B, 
it collapses to 42.0\% accuracy due to the weak System 1 
model, while our continuous interpolation maintains robust 
performance. Static merging methods (TA, TIES, MI-$\lambda$) 
cannot adapt to varying difficulty distributions. While MI-07 
achieves strong results on Qwen3-4B, our dynamic approach 
consistently outperforms the best static configuration across 
benchmarks.

\noindent \textbf{Analysis of Thinking Ratio.}
The Thinking Ratio reveals adaptive resource allocation. On 
simpler benchmarks like GSM8K, DAMI-Pref and DAMI-Conf activate 
deep thinking for only 25-41\% of queries on Qwen3-4B. As 
difficulty increases toward AIME24/25, the ratio rises to 
79-95\%. This adaptive behavior distinguishes our approach 
from static methods, which either overthink simple queries 
(Thinking, TA) or underthink difficult ones (TIES, MI-03). 
The smooth variation demonstrates that DAMI successfully 
calibrates reasoning depth to query complexity.

\subsection{Analysis}

\noindent \textbf{Ablation study.} Figure~\ref{fig:ablation} presents ablation results on the Qwen3-4B 
model pair, examining the contribution of key design components.
For DAMI-Pref, removing the cost-based tiebreaker (-Cost) slightly 
improves accuracy but increases token 
consumption by 18.0\%. This confirms that 
the hierarchical preference criterion effectively balances 
accuracy and efficiency.
For DAMI-Conf, both signals contribute to the final performance. 
Removing Cognitive Discrepancy ($-S_{dis}$) causes a 3.7\% 
accuracy drop, while removing Holistic Ambiguity ($-S_{ambi}$) 
results in a 1.9\% drop. Using only the Thinking model's 
confidence ($\mathcal{C}^T(q)$-only) degrades accuracy to 
84.5\%, demonstrating that the inter-model comparison is 
essential for reliable difficulty estimation. More detailed results are provided in Table~\ref{ablation} (Appendix).  

\noindent \textbf{Training data efficiency of DAMI-Pref.}
Figure~\ref{fig:data_efficiency} compares DAMI-Pred and DAMI-Pref 
under varying training data sizes. DAMI-Pref achieves near-optimal performance 
with only 1K samples and shows minimal improvement 
beyond, demonstrating strong robustness to limited 
supervision. In contrast, DAMI-Pred requires significantly more 
data to stabilize: it does not surpass the Thinking baseline until 4K samples. This gap 
confirms our motivation for preference-based estimation: 
pairwise comparisons provide more robust learning signals than 
pointwise regression, particularly in low-data regimes where 
label noise has greater impact. 
%For practical deployment, this 
% means DAMI-Pref can be effectively trained with minimal annotation 
% effort, making it more accessible for real-world applications.

\noindent \textbf{Case study.} 
Figure~\ref{fig:case} compares responses on a geometry problem 
from MATH-500. The Thinking model produces 2,846 tokens with 
typical overthinking patterns. DEER reduces output to 1,188 tokens through early termination, but cannot skip the lengthy content before the exit marker. This reflects a fundamental limitation of output control methods. LLM Routing, constrained to discrete model 
selection, routes this query to Thinking and yields 3,017 tokens.
In contrast, DAMI generates only 351 tokens with qualitatively different behavior: direct 
problem formulation, concise calculation, and no redundant 
verification. This illustrates the core advantage of capability 
control: DAMI produces 
inherently efficient reasoning by modulating cognitive 
configuration.

\noindent \textbf{Generalization on Multimodal Tasks.}  Beyond text-only models, we evaluate DAMI on the multimodal model Qwen3-VL-8B \citep{Qwen3-VL}. We curate a diverse evaluation suite comprising 5 video understanding benchmarks (VideoHolmes, MMVU, VideoMMMU, VideoMME, LVBench) and 4 image-text reasoning benchmarks (WeMath, MathVista, MMMU, MathVision). As shown in Figure 7, DAMI-Conf consistently outperforms all baselines across both modalities. On video understanding tasks, DAMI-Conf achieves 64.0\% average accuracy, surpassing the Thinking model by 2.5\% absolute. On image-text reasoning tasks, DAMI-Conf attains 66.5\%, a 1.6\% improvement over Thinking. These findings confirm that our DAMI framework generalizes effectively to widely multimodal tasks. Moreover, it reveals that DAMI not only enables efficient reasoning, but also effectively integrates the complementary knowledge and reasoning capabilities of both models, leading to consistent performance gains across diverse benchmarks.

\noindent \textbf{Computational cost.} Figure~\ref{fig:computational_cost} presents end-to-end latency. Despite introducing additional steps for $\lambda$ 
prediction and parameter merging, both methods achieve substantial 
speedups over vanilla Thinking. The overhead from $\lambda$ 
prediction is modest: DAMI-Pref only requires scoring each 
candidate $\lambda$ with a lightweight reward model, while DAMI-Conf  extracts confidence from short answer-only generations 
rather than full reasoning chains. Parameter merging adds $\sim$5.3s 
when loading models on-demand, but this can be reduced to under 1s 
by pre-loading both checkpoints into GPU memory. These 
overheads are negligible compared to generation time, which 
dominates total latency. By reducing output length by 35\% (DAMI-Pref) 
and 56\% (DAMI-Conf), our methods achieve 1.46$\times$ and 
1.86$\times$ end-to-end speedups respectively.

\section{Conclusion}

% This paper challenges the prevailing output control paradigm 
% for efficient reasoning. We demonstrate that the root cause of 
% overthinking lies in the model's cognitive configuration, 
% not its output. By shifting to capability control through 
% parameter interpolation, we reveal a smooth and predictable 
% Pareto frontier governed by the monotonicity property of 
% Reasoning Intensity.
% Based on this insight, our DAMI framework dynamically 
% estimates query-specific $\lambda(q)$, with preference learning 
% and confidence-based methods providing robust solutions for 
% diverse deployment scenarios. Experiments validate that 
% capability control consistently outperforms output control 
% in both accuracy and efficiency.
% We believe this work offers a new lens for the community: 
% efficient reasoning should not be about limiting what models 
% say, but about configuring how they think. We hope DAMI 
% inspires further research into continuous, parameter-level 
% mechanisms for adaptive inference.

This paper addresses the challenge that System 1's efficiency and System 2's reasoning capability cannot be readily unified within a single model. Unlike existing methods that constrain what models produce (output control), we propose capability control, which modulates how models think through dynamic parameter interpolation. Our DAMI framework offers two complementary methods, training based DAMI-Pref and training free DAMI-Conf, to achieve query-adaptive reasoning intensity.
Beyond the technical contributions, this work suggests a broader perspective: the dichotomy between fast and slow thinking need not be a binary choice resolved at training time. The monotonicity and continuity properties we establish indicate that reasoning depth exists on a smooth spectrum, accessible through simple linear operations in parameter space. This finding opens possibilities for more flexible cognitive architectures that adapt computational budget.

\bibliographystyle{named}
\bibliography{ijcai25}

\appendix

\newtheorem{lemma}{Lemma}
\newtheorem{corollary}{Corollary}
\newtheorem{definition}{Definition}
\newtheorem{proposition}{Proposition}
\newtheorem{assumption}{Assumption}
\newtheorem{remark}{Remark}

\newpage

\section{Theoretical Analysis}
\label{appendix:theory}

In this appendix, we provide theoretical justifications for the three empirical findings in Section 3.1. We establish that (1) linear interpolation between models satisfying Linear Mode Connectivity does not cross loss barriers, (2) performance monotonicity holds under mild assumptions about capability ordering and interpolation linearity, and (3) representation continuity follows from the Lipschitz properties of transformer networks.
\subsection{Interpolation Validity under Linear Mode Connectivity}
\label{appendix:lmc}

We establish why linear parameter interpolation between the 
Instruct and Thinking models yields valid, well-behaved models 
rather than degenerate solutions. 

\subsubsection{Theoretical Background}

Neural network loss landscapes are highly non-convex, containing 
multiple local minima separated by regions of high loss. When 
linearly interpolating between two model parameters, a natural 
concern is whether the interpolation path crosses such high-loss 
regions, resulting in degenerate intermediate models \citep{yao2025vecinferefficientllminference,yang2023multileveladaptivecontrastivelearning}.

\begin{definition}[Loss Barrier]
\label{def:barrier}
Given two parameter configurations $\Theta^{(1)}$ and $\Theta^{(2)}$, 
a \emph{loss barrier} of height $\epsilon$ exists on the linear 
path between them if:
\begin{equation}
\max_{\lambda \in [0,1]} \mathcal{L}(\Theta_\lambda) > 
\max\{\mathcal{L}(\Theta^{(1)}), \mathcal{L}(\Theta^{(2)})\} + \epsilon
\end{equation}
where $\Theta_\lambda = \lambda \Theta^{(2)} + (1-\lambda) \Theta^{(1)}$ 
and $\mathcal{L}(\cdot)$ denotes the task loss.
\end{definition}

\begin{definition}[Linear Mode Connectivity]
\label{def:lmc}
Two models $\Theta^{(1)}$ and $\Theta^{(2)}$ satisfy \emph{Linear 
Mode Connectivity (LMC)} if no significant loss barrier exists on 
the linear path between them, i.e., for all $\lambda \in [0,1]$:
\begin{equation}
\mathcal{L}(\Theta_\lambda) \leq 
\max\{\mathcal{L}(\Theta^{(1)}), \mathcal{L}(\Theta^{(2)})\} + \epsilon
\end{equation}
where $\epsilon$ is a small tolerance threshold.
\end{definition}

The concept of Linear Mode Connectivity was formally introduced 
by \citet{frankle2020linearmodeconnectivitylotteryLMC}, who investigated conditions under 
which neural networks trained with different random seeds remain 
linearly connected in parameter space \citep{chen2024multitaskroleplayingagentcapable,yang2025breakingtradeofffaithfulnessexpressiveness}.

\subsubsection{Existing Theoretical Results}

A crucial result for our work comes from \citet{neyshabur2020what}, 
who studied the properties of models fine-tuned from pre-trained 
checkpoints:

\begin{theorem}[Same-Basin Property of Fine-tuned Models, 
\citealt{neyshabur2020what}]
\label{thm:neyshabur}
Let $\Theta^{(0)}$ be a pre-trained model, and let $\Theta^{(1)}$, 
$\Theta^{(2)}$ be obtained by fine-tuning $\Theta^{(0)}$ on 
(possibly different) downstream tasks or with different 
hyperparameters. Then:
\begin{enumerate}
    \item Both $\Theta^{(1)}$ and $\Theta^{(2)}$ remain in the 
    same basin of the loss landscape as $\Theta^{(0)}$.
    \item The two fine-tuned models are close in parameter space 
    and similar in feature space.
\end{enumerate}
\end{theorem}

This theorem has an important implication for linear interpolation:

\begin{corollary}[LMC for Co-originated Models]
\label{cor:lmc_coorigin}
If $\Theta^{(1)}$ and $\Theta^{(2)}$ are both fine-tuned from the 
same pre-trained checkpoint $\Theta^{(0)}$, then they satisfy 
Linear Mode Connectivity.
\end{corollary}

\begin{proof}
By Theorem \ref{thm:neyshabur}, both $\Theta^{(1)}$ and $\Theta^{(2)}$ 
reside in the same optimization basin $\mathcal{B}$ containing 
$\Theta^{(0)}$. Since an optimization basin is a connected region 
of low loss, and both endpoints lie within $\mathcal{B}$, the 
linear path between them does not exit $\mathcal{B}$ (assuming 
$\mathcal{B}$ is approximately convex in the local neighborhood). 
Therefore, no significant loss barrier exists along the 
interpolation path, satisfying the definition of LMC.
\end{proof}

\subsubsection{Application to Our Setting}

We now apply these theoretical results to our specific context.

\begin{proposition}[Interpolation Validity for Instruct-Thinking Pairs]
\label{prop:our_setting}
Let $\Theta^{(\text{I})}$ (Instruct) and $\Theta^{(\text{T})}$ 
(Thinking) be two models derived from the same base pre-trained 
model through different post-training procedures (instruction 
tuning and reasoning enhancement, respectively). Then for any 
$\lambda \in [0,1]$, the interpolated model 
$\Theta^{(\text{M})}_\lambda = \lambda \Theta^{(\text{T})} + 
(1-\lambda) \Theta^{(\text{I})}$ satisfies:
\begin{equation}
\mathcal{L}(\Theta^{(\text{M})}_\lambda) \leq 
\max\{\mathcal{L}(\Theta^{(\text{I})}), 
\mathcal{L}(\Theta^{(\text{T})})\} + \epsilon
\end{equation}
\end{proposition}

\begin{proof}
The Instruct and Thinking models in modern LLM families (e.g., 
Qwen, DeepSeek) are typically produced through the following 
pipeline:
\begin{enumerate}
    \item A base model $\Theta^{(0)}$ is pre-trained on large-scale 
    text corpora.
    \item The Instruct model $\Theta^{(\text{I})}$ is obtained by 
    fine-tuning $\Theta^{(0)}$ with instruction-following data.
    \item The Thinking model $\Theta^{(\text{T})}$ is obtained by 
    further training (from $\Theta^{(0)}$ or $\Theta^{(\text{I})}$) 
    with reasoning-enhanced data and reinforcement learning.
\end{enumerate}

In either case, both models originate from the same pre-trained 
checkpoint $\Theta^{(0)}$. By Theorem \ref{thm:neyshabur} and 
Corollary \ref{cor:lmc_coorigin}, they satisfy LMC, which directly 
implies the bounded loss property.
\end{proof}

\subsection{Sufficient Conditions for Performance Monotonicity}
\label{appendix:monotonicity}

We establish theoretical conditions under which increasing the 
reasoning intensity $\lambda$ leads to monotonically improving 
accuracy.

\begin{assumption}[Capability Ordering]
\label{asp:capability}
For any query $q$ from the task distribution $\mathcal{Q}$, let 
$p^{(\text{I})}(q) \in [0,1]$ and $p^{(\text{T})}(q) \in [0,1]$ 
denote the probability of correctly answering $q$ by the Instruct 
and Thinking models respectively. We assume:
\begin{equation}
p^{(\text{T})}(q) \geq p^{(\text{I})}(q), \quad \forall q \in \mathcal{Q}
\end{equation}
\end{assumption}

\begin{remark}
This assumption reflects the empirical observation that System 2 
(Thinking) models, equipped with extended chain-of-thought reasoning, 
generally achieve higher or equal accuracy compared to System 1 
(Instruct) models on reasoning tasks. The assumption allows equality 
for simple queries where both models succeed.
\end{remark}

\begin{assumption}[Bounded Interpolation Deviation]
\label{asp:deviation}
The success probability of the interpolated model deviates from 
linear interpolation by a bounded amount:
\begin{equation}
p^{(\text{M})}_\lambda(q) = \lambda \cdot p^{(\text{T})}(q) + 
(1-\lambda) \cdot p^{(\text{I})}(q) + \delta_\lambda(q)
\end{equation}
where the deviation term $\delta_\lambda(q)$ satisfies:
\begin{enumerate}
    \item \textbf{Boundedness}: $|\delta_\lambda(q)| \leq \delta_{max}$ 
    for all $\lambda \in [0,1]$
    \item \textbf{Lipschitz continuity}: 
    $|\delta_\lambda(q) - \delta_{\lambda'}(q)| \leq L_\delta \cdot |\lambda - \lambda'|$ 
    for all $\lambda, \lambda' \in [0,1]$
\end{enumerate}
\end{assumption}

\begin{remark}
The Lipschitz condition ensures that the deviation from linearity 
changes smoothly with $\lambda$. This is justified by the 
representation continuity established in Section \ref{appendix:lipschitz}: 
since model representations change smoothly with $\lambda$, the 
resulting success probabilities should also change smoothly.
\end{remark}

\begin{theorem}[Monotonicity of Expected Accuracy]
\label{thm:monotonicity}
Under Assumptions \ref{asp:capability} and \ref{asp:deviation}, 
if $L_\delta < \mathbb{E}_{q \sim \mathcal{Q}}[p^{(\text{T})}(q) - 
p^{(\text{I})}(q)]$, then the expected accuracy is monotonically 
non-decreasing in $\lambda$:
\begin{equation}
\frac{\partial}{\partial \lambda} \mathbb{E}_{q \sim \mathcal{Q}}
[p^{(\text{M})}_\lambda(q)] \geq 0
\end{equation}
\end{theorem}

\begin{proof}
By Assumption \ref{asp:deviation}, the success probability can 
be decomposed as:
\begin{equation}
p^{(\text{M})}_\lambda(q) = \lambda \cdot p^{(\text{T})}(q) + 
(1-\lambda) \cdot p^{(\text{I})}(q) + \delta_\lambda(q)
\end{equation}

Taking the derivative with respect to $\lambda$:
\begin{equation}
\frac{\partial}{\partial \lambda} p^{(\text{M})}_\lambda(q) = 
p^{(\text{T})}(q) - p^{(\text{I})}(q) + 
\frac{\partial \delta_\lambda(q)}{\partial \lambda}
\end{equation}

By the Lipschitz condition in Assumption \ref{asp:deviation}, 
where $\delta_\lambda(q)$ is $L_\delta$-Lipschitz in $\lambda$, 
we have:
\begin{equation}
\left|\frac{\partial \delta_\lambda(q)}{\partial \lambda}\right| 
\leq L_\delta
\end{equation}

Taking expectation over $q \sim \mathcal{Q}$:
\begin{align}
\frac{\partial}{\partial \lambda} \mathbb{E}_{q}[p^{(\text{M})}_\lambda(q)] 
&= \mathbb{E}_{q}\left[p^{(\text{T})}(q) - p^{(\text{I})}(q) + 
\frac{\partial \delta_\lambda(q)}{\partial \lambda}\right] \\
&\geq \mathbb{E}_{q}[p^{(\text{T})}(q) - p^{(\text{I})}(q)] - L_\delta
\end{align}

Under the condition $L_\delta < \mathbb{E}_{q}[p^{(\text{T})}(q) - 
p^{(\text{I})}(q)]$:
\begin{equation}
\frac{\partial}{\partial \lambda} \mathbb{E}_{q}[p^{(\text{M})}_\lambda(q)] 
\geq \mathbb{E}_{q}[p^{(\text{T})}(q) - p^{(\text{I})}(q)] - L_\delta > 0
\end{equation}
\end{proof}

\begin{corollary}[Exact Monotonicity under Linear Interpolation]
\label{cor:exact_mono}
If the interpolation is exactly linear, i.e., $\delta_\lambda(q) = 0$ 
for all $\lambda$ and $q$, then:
\begin{equation}
\frac{\partial}{\partial \lambda} p^{(\text{M})}_\lambda(q) = 
p^{(\text{T})}(q) - p^{(\text{I})}(q) \geq 0
\end{equation}
with strict inequality when $p^{(\text{T})}(q) > p^{(\text{I})}(q)$.
\end{corollary}

\begin{proof}
Setting $\delta_\lambda(q) = 0$ in the general formula yields 
$\frac{\partial}{\partial \lambda} p^{(\text{M})}_\lambda(q) = 
p^{(\text{T})}(q) - p^{(\text{I})}(q)$, which is non-negative 
by Assumption \ref{asp:capability}.
\end{proof}

\paragraph{Direct Verification of Monotonicity.}
Figure \ref{fig-pilot}(a) provides direct evidence: accuracy 
increases monotonically with $\lambda$, consistent with the 
prediction of Theorem \ref{thm:monotonicity}.

\subsection{Representation Continuity via Lipschitz Analysis}
\label{appendix:lipschitz}

We explain why interpolation induces continuous trajectories in 
representation space, as observed in Figure \ref{fig-pilot}(b). 
The key insight is that transformer networks are Lipschitz 
continuous with respect to their parameters.

\begin{definition}[Lipschitz Continuity]
\label{def:lipschitz}
A function $f: \mathcal{X} \rightarrow \mathcal{Y}$ is 
$L$-Lipschitz if for all $x_1, x_2 \in \mathcal{X}$:
\begin{equation}
\|f(x_1) - f(x_2)\| \leq L \cdot \|x_1 - x_2\|
\end{equation}
\end{definition}

For a parametric function $f(x; \theta)$, we distinguish two 
types of Lipschitz continuity:
\begin{itemize}
    \item \textbf{Input-Lipschitz}: $\|f(x_1; \theta) - f(x_2; \theta)\| 
    \leq L^{(x)} \|x_1 - x_2\|$
    \item \textbf{Parameter-Lipschitz}: $\|f(x; \theta_1) - f(x; \theta_2)\| 
    \leq L^{(\theta)} \|x\| \cdot \|\theta_1 - \theta_2\|$
\end{itemize}

\subsubsection{Lipschitz Properties of Transformer Components}

\begin{lemma}[Linear Layer]
\label{lem:linear}
For a linear transformation $f(x; W) = Wx$:
\begin{enumerate}
    \item Input-Lipschitz with $L^{(x)} = \|W\|$ (spectral norm)
    \item Parameter-Lipschitz: $\|W_1 x - W_2 x\| \leq \|x\| \cdot \|W_1 - W_2\|$
\end{enumerate}
\end{lemma}

\begin{proof}
For input-Lipschitz: $\|Wx_1 - Wx_2\| = \|W(x_1-x_2)\| \leq \|W\| \cdot \|x_1-x_2\|$.

For parameter-Lipschitz: $\|W_1 x - W_2 x\| = \|(W_1-W_2)x\| \leq \|W_1-W_2\| \cdot \|x\|$.
\end{proof}

\begin{lemma}[Activation Functions]
\label{lem:activation}
Common activation functions are Lipschitz continuous:
\begin{itemize}
    \item ReLU: 1-Lipschitz
    \item GeLU: approximately 1.13-Lipschitz
    \item Softmax: 1-Lipschitz (with respect to $\ell_1$ norm)
\end{itemize}
\end{lemma}

\begin{lemma}[Transformer Layer]
\label{lem:transformer}
A transformer layer $f_l(h; \theta_l)$ consisting of multi-head 
attention and feed-forward sublayers satisfies:
\begin{enumerate}
    \item Input-Lipschitz with constant $L^{(h)}_l$
    \item Parameter-Lipschitz: for bounded input $\|h\| \leq B$,
    \begin{equation}
    \|f_l(h; \theta_{l,1}) - f_l(h; \theta_{l,2})\| \leq 
    L^{(\theta)}_l \cdot B \cdot \|\theta_{l,1} - \theta_{l,2}\|
    \end{equation}
\end{enumerate}
where $L^{(h)}_l$ and $L^{(\theta)}_l$ depend on the layer 
architecture and weight magnitudes.
\end{lemma}

\begin{proof}[Proof Sketch]
The transformer layer is a composition of linear transformations, 
attention mechanisms, and activation functions. By the chain rule 
for Lipschitz functions (composition of $L_1$-Lipschitz and 
$L_2$-Lipschitz functions is $L_1 L_2$-Lipschitz), and noting 
that each component is Lipschitz (Lemmas \ref{lem:linear} and 
\ref{lem:activation}), the overall layer inherits Lipschitz 
continuity with respect to both inputs and parameters.
\end{proof}

\begin{theorem}[Representation Continuity]
\label{thm:rep_continuity}
Let $h^{(L)}_\lambda(x)$ denote the final-layer representation of 
input $x$ under the interpolated model $\Theta^{(\text{M})}_\lambda$. 
Assume that intermediate representations are bounded: 
$\|h^{(l)}_\lambda(x)\| \leq B$ for all $l$ and $\lambda$. Then 
there exists a constant $C > 0$ such that for any 
$\lambda, \lambda' \in [0,1]$:
\begin{equation}
\|h^{(L)}_\lambda(x) - h^{(L)}_{\lambda'}(x)\| \leq 
C \cdot |\lambda - \lambda'|
\end{equation}
That is, the representation is Lipschitz continuous with respect 
to the interpolation coefficient $\lambda$.
\end{theorem}

\begin{proof}
We proceed by analyzing the layer-wise propagation of differences.

\textbf{Step 1: Parameter Difference.}
For the interpolated model, the parameters at layer $l$ are:
\begin{equation}
\theta_{l,\lambda} = \lambda \theta^{(\text{T})}_l + 
(1-\lambda) \theta^{(\text{I})}_l
\end{equation}
The parameter difference between two interpolation coefficients is:
\begin{equation}
\|\theta_{l,\lambda} - \theta_{l,\lambda'}\| = 
|\lambda - \lambda'| \cdot \|\theta^{(\text{T})}_l - \theta^{(\text{I})}_l\|
\end{equation}

Let $\Delta_l = \|\theta^{(\text{T})}_l - \theta^{(\text{I})}_l\|$ 
denote the parameter distance at layer $l$.

\textbf{Step 2: Layer-wise Recursion.}
Let $D_l = \|h^{(l)}_\lambda(x) - h^{(l)}_{\lambda'}(x)\|$ denote 
the representation difference at layer $l$. Using the triangle 
inequality:
\begin{align}
D_l &= \|f_l(h^{(l-1)}_\lambda; \theta_{l,\lambda}) - 
f_l(h^{(l-1)}_{\lambda'}; \theta_{l,\lambda'})\| \\
&\leq \underbrace{\|f_l(h^{(l-1)}_\lambda; \theta_{l,\lambda}) - 
f_l(h^{(l-1)}_\lambda; \theta_{l,\lambda'})\|}_{\text{(A): parameter difference}} \\
&\quad + \underbrace{\|f_l(h^{(l-1)}_\lambda; \theta_{l,\lambda'}) - 
f_l(h^{(l-1)}_{\lambda'}; \theta_{l,\lambda'})\|}_{\text{(B): input difference}}
\end{align}

By Lemma \ref{lem:transformer}:
\begin{itemize}
    \item Term (A) $\leq L^{(\theta)}_l \cdot B \cdot 
    \|\theta_{l,\lambda} - \theta_{l,\lambda'}\| = 
    L^{(\theta)}_l \cdot B \cdot |\lambda - \lambda'| \cdot \Delta_l$
    \item Term (B) $\leq L^{(h)}_l \cdot D_{l-1}$
\end{itemize}

This yields the recursion:
\begin{equation}
D_l \leq L^{(\theta)}_l B \Delta_l \cdot |\lambda - \lambda'| + 
L^{(h)}_l \cdot D_{l-1}
\label{eq:recursion}
\end{equation}

\textbf{Step 3: Solving the Recursion.}
Let $\alpha_l = L^{(\theta)}_l B \Delta_l$ and $\beta_l = L^{(h)}_l$ 
for notational simplicity. The recursion becomes:
\begin{equation}
D_l \leq \alpha_l \cdot |\lambda - \lambda'| + \beta_l \cdot D_{l-1}
\end{equation}

For the base case, we consider the embedding layer. If the embedding 
parameters are also interpolated, then $D_0 = \alpha_0 \cdot |\lambda - \lambda'|$ 
for some $\alpha_0$. Otherwise, $D_0 = 0$.

Unrolling the recursion:
\begin{align}
D_L &\leq \alpha_L |\lambda - \lambda'| + \beta_L D_{L-1} \\
&\leq \alpha_L |\lambda - \lambda'| + \beta_L 
(\alpha_{L-1} |\lambda - \lambda'| + \beta_{L-1} D_{L-2}) \\
&= (\alpha_L + \beta_L \alpha_{L-1}) |\lambda - \lambda'| + 
\beta_L \beta_{L-1} D_{L-2} \\
&\leq \ldots \\
&\leq \left( \sum_{l=1}^{L} \alpha_l \prod_{j=l+1}^{L} \beta_j \right) 
|\lambda - \lambda'| + \left(\prod_{j=1}^{L} \beta_j\right) D_0
\end{align}

\textbf{Step 4: Final Bound.}
Defining:
\begin{equation}
C = \sum_{l=0}^{L} \alpha_l \prod_{j=l+1}^{L} \beta_j = 
\sum_{l=0}^{L} L^{(\theta)}_l B \Delta_l \prod_{j=l+1}^{L} L^{(h)}_j
\end{equation}
we obtain:
\begin{equation}
D_L = \|h^{(L)}_\lambda(x) - h^{(L)}_{\lambda'}(x)\| \leq C \cdot |\lambda - \lambda'|
\end{equation}

Since all terms ($L^{(\theta)}_l$, $L^{(h)}_l$, $B$, $\Delta_l$) are 
finite constants, $C$ is a finite constant, establishing Lipschitz 
continuity.
\end{proof}

\begin{remark}[Interpretation of the Constant $C$]
The constant $C$ has a clear structure:
\begin{equation}
C = \sum_{l=0}^{L} \underbrace{L^{(\theta)}_l B \Delta_l}_{\text{direct effect at layer } l} 
\times \underbrace{\prod_{j=l+1}^{L} L^{(h)}_j}_{\text{amplification through subsequent layers}}
\end{equation}
Each layer contributes a term proportional to its parameter distance 
$\Delta_l$, amplified by the Lipschitz constants of all subsequent 
layers. This explains why changes in middle-to-deep layers (where 
$\Delta_l$ is larger, as shown in Figure \ref{fig-pilot}(c)) have 
stronger effects on the final representation.
\end{remark}

\begin{corollary}[Continuous Trajectory in Activation Space]
\label{cor:continuous}
The mapping $\lambda \mapsto h^{(L)}_\lambda(x)$ is continuous. 
Moreover, the set $\{h^{(L)}_\lambda(x) : \lambda \in [0,1]\}$ 
forms a connected path in activation space.
\end{corollary}

\begin{proof}
Lipschitz continuity implies uniform continuity, which implies 
continuity. The image of a connected set $[0,1]$ under a continuous 
function is connected.
\end{proof}

\subsubsection{Connection to Empirical Observations}

Theorem \ref{thm:rep_continuity} explains the PCA visualization 
in Figure \ref{fig-pilot}(b):

\begin{enumerate}
    \item \textbf{Continuous Trajectory}: The Lipschitz bound ensures 
    that adjacent $\lambda$ values produce nearby representations, 
    resulting in the smooth curve observed in PCA space.
    
    \item \textbf{Correlation with $\lambda$}: Since representation 
    change is bounded by $C \cdot |\lambda - \lambda'|$, the 
    dominant direction of variation aligns with the interpolation 
    axis. This explains the high correlation ($r = 0.974$) between 
    PC1 and $\lambda$.
    
    \item \textbf{Proportional Spacing}: The roughly equal spacing 
    between consecutive $\lambda$ values in PCA space reflects the 
    linear relationship between $|\lambda - \lambda'|$ and 
    representation distance.
\end{enumerate}

\subsection{Benchmark Details.}

\paragraph{\textsc{Mathematical Benchmarks}}.
We utilize a range of benchmark datasets to thoroughly evaluate model capabilities on mathematical reasoning tasks. These benchmarks cover a broad spectrum of difficulty levels, ranging from basic arithmetic operations to challenging competition-grade mathematics.
\begin{itemize}
\item \textbf{GSM8K} comprises 1,319 meticulously selected grade-school mathematics problems, intended to assess models' multi-step reasoning abilities in elementary mathematical contexts. Each problem generally involves two to eight successive operations, primarily grounded in fundamental arithmetic, and necessitates precise computation of intermediate values.
\item \textbf{MATH-500} presents a demanding set of high-school-level problems spanning various mathematical domains, such as Prealgebra, Algebra, and Number Theory. These problems, largely derived from mathematics competitions, require both abstract reasoning and sophisticated logical inference. For consistency with previous studies, we use the 500-problem subset originally established by OpenAI as the evaluation standard.
\item \textbf{AMC 2023} consists of 40 problems drawn from the 2023 American Mathematics Competitions (AMC), an annual event hosted by the Mathematical Association of America (MAA) that seeks to cultivate problem-solving abilities and discover mathematically talented students. The problems encompass fundamental areas including algebra, geometry, number theory, and combinatorics, providing a rigorous benchmark for assessing sophisticated mathematical reasoning.
\item \textbf{AIME 2024 \& 2025} include 60 challenging problems from the 2024 and 2025 editions of the American Invitational Mathematics Examination (AIME). This highly regarded competition tests participants' mathematical reasoning across a wide array of topics, encompassing arithmetic, algebra, counting, geometry, number theory, probability, and various other secondary school mathematics subjects.
\end{itemize}

\paragraph{\textsc{Video Reasoning Benchmarks}}.
To assess the model's ability to perform high-level logic, deduction, and expert-level analysis within video contexts, we employ the following benchmarks:
\begin{itemize}
\item \textbf{Holmes} (Video-Holmes) evaluates complex reasoning and active clue-seeking abilities within the mystery and detective genre. It features suspenseful clips where models must perform multi-hop deduction, detect anomalies, and integrate temporal clues to identify perpetrators and motives, simulating high-level cognitive deduction.
\item \textbf{MMVU} is designed to assess expert-level disciplinary knowledge and reasoning in video understanding. It comprises 3,000 expert-annotated questions spanning four core disciplines—Science, Medicine, Humanities, and Engineering—requiring models to apply textbook-level academic knowledge to solve complex problems.
\item \textbf{Video-MMMU} focuses on the capability of knowledge acquisition and transfer. Utilizing educational videos and lectures, it evaluates the "learning-to-answer" paradigm, measuring how well a model can perceive, understand, and apply new information learned from video content to solve novel questions.
\end{itemize}

\paragraph{\textsc{General Video Understanding Benchmarks}}.
We further evaluate the model's robustness in processing diverse video durations and managing long-term contexts through these comprehensive benchmarks:
\begin{itemize}
\item \textbf{Video-MME} is a full-spectrum benchmark designed to assess multimodal analysis across varying durations, ranging from short clips to hour-long videos. It covers diverse domains such as movies and sports, integrating visual, audio, and subtitle modalities to test temporal reasoning and cross-modal alignment.
\item \textbf{LVBench} targets extreme long-video understanding with high-quality content averaging over an hour in duration. It specifically challenges models' long-term memory and information retrieval capabilities, requiring the extraction of specific entities and events from extensive and redundant temporal contexts.
\end{itemize}

\paragraph{\textsc{Multimodal Mathematical Reasoning Benchmarks}}.
To rigorously evaluate the intersection of visual perception and mathematical logic, we utilize datasets that range from diagnostic tests to expert-level challenges:
\begin{itemize}
\item \textbf{MMMU} evaluates expert-level multimodal reasoning across 30 heterogeneous disciplines, including engineering and clinical medicine. It demands deep domain knowledge to interpret complex visual data—such as chemical structures and medical imaging—combined with sophisticated logical inference.
\item \textbf{MathVista} serves as an omnibus benchmark aggregating diverse mathematical tasks from 28 existing sources. It provides a holistic assessment of visual mathematical reasoning across a wide range of visual contexts, including geometry, function plots, and statistical charts.
\item \textbf{We-Math} is a visual-centric benchmark constructed to address visual redundancy in existing datasets. It emphasizes fine-grained evaluation of multi-step reasoning where visual information is indispensable, serving as a diagnostic tool for true visual grounding in mathematical logic.
\item \textbf{MathVision} is a large-scale, high-quality benchmark sourced from real-world mathematics competitions. Distinguished by rigorous data decontamination and high difficulty, it pushes the upper limits of current models by testing robust visual reasoning in complex, competition-grade scenarios.
\end{itemize}

\section{Hyperparameter Sensitivity Experiments}
\label{hyperparameter_sensitivity}

\noindent \textbf{Hyperparameter Sensitivity of DAMI-Conf.}
Figure~\ref{fig:hyperparameter} investigates the effect of 
decision boundary $\mu$ and temperature $\tau$ on DAMI-Conf 
performance. The heatmap reveals that DAMI-Conf is robust across 
a wide range of hyperparameter configurations. The optimal performance is achieved at $\mu$=0.3 
and $\tau$=0.3, which we adopt as default settings. 
Notably, the central region ($\mu \in [0.2, 0.4]$, 
$\tau \in [0.2, 0.4]$) consistently achieves above 89.5\% 
accuracy, while extreme values at the boundaries lead to 
moderate degradation. This pattern is intuitive: overly small 
$\mu$ biases toward the Thinking model regardless of query 
difficulty, while overly large $\mu$ favors the Instruct model 
even for challenging queries. Similarly, extreme $\tau$ values 
either over-smooth or over-sharpen the $\lambda$ distribution. 
The overall stability across configurations confirms that 
DAMI-Conf does not require careful hyperparameter tuning, 
making it practical for deployment without extensive validation.

\begin{figure}[!t]
  \centerline{\includegraphics[scale=0.5]{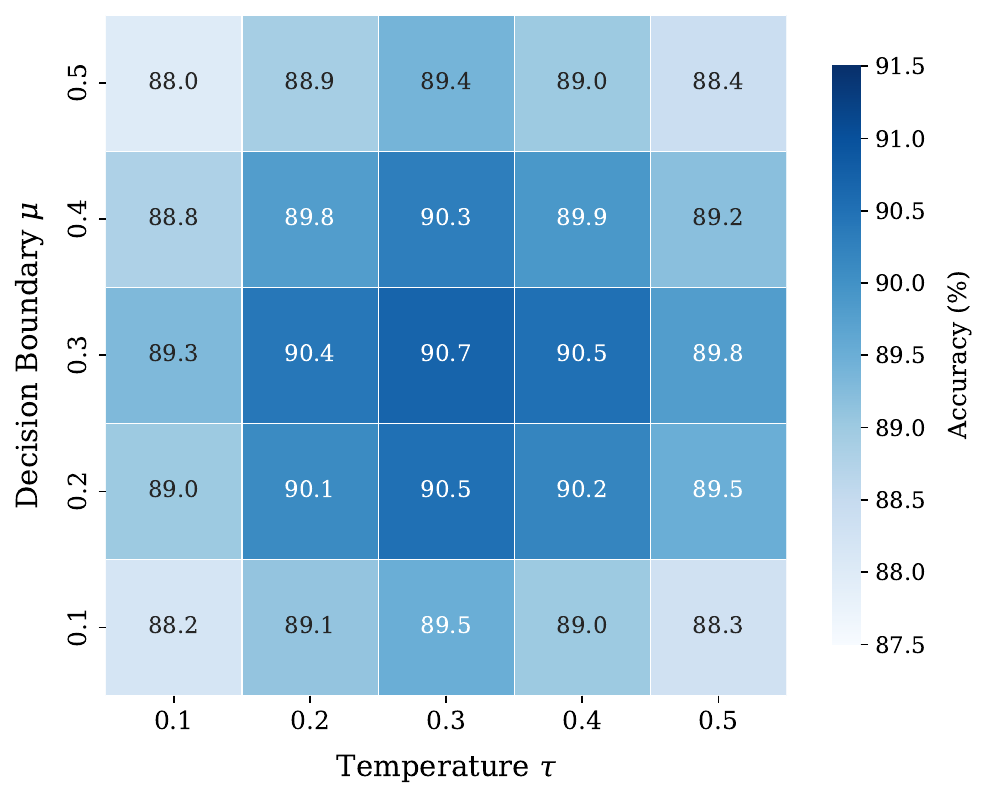}}
  \caption{Investigation on the effect of hyperparameter $\mu$ and $\tau$.}
  \label{fig:hyperparameter}
\end{figure}

{
\setlength{\tabcolsep}{2.5pt}
\begin{table*}[h!]
\centering
\scalebox{0.88}{
\begin{tabular}{@{}lcccccccccccccccccc@{}} 
\toprule
 \multirow{3}{*}{\textbf{Method}} 
 & \multicolumn{3}{c}{\textbf{GSM8K}}& \multicolumn{3}{c}{\textbf{MATH-500}} & \multicolumn{3}{c}{\textbf{AMC}} & \multicolumn{3}{c}{\textbf{AIME24}}  & \multicolumn{3}{c}{\textbf{AIME25}}  & \multicolumn{3}{c}{\textbf{Overall}}
 \\
   & {Acc$\uparrow$} & {Tok$\downarrow$} & {Think} & {Acc$\uparrow$} & {Tok$\downarrow$} & {Think} & Acc$\uparrow$ & Tok$\downarrow$ & {Think} & {Acc}$\uparrow$  & {Tok$\downarrow$} & {Think} & {Acc$\uparrow$} & {Tok$\downarrow$} & {Think} & {Acc$\uparrow$} & {Tok$\downarrow$} &{Think}  \\ 

   & {2$\times$} & \#N & \#R & {2$\times$} & \#N & \#R & {4$\times$} & \#N & \#R & {4$\times$}  & \#N & \#R & {4$\times$} & \#N & \#R & - & \#N & \#R  \\ 
   
 \midrule

  \textit{DAMI-Pref} & 94.5 & 843 & 25 & {96.4} & 4315 & 64 & {98.8} & 7872 & 86 & {85.0} & 15412 & 95 & {83.3} & 17488 & 92 & {91.6} &
  9186 & 72 \\

 w/o Cost & {94.0} & 1123 & 45 & {96.6} & 5102 & 78 & {98.8} & 8943 & 92 & {85.0} & 17821 & 98 & {83.3} & 19234 & 96 & {91.5} & 10843 & 82 \\

 \midrule

\textit{DAMI-Conf}  & 94.5 & 856 & 41 & 96.2 & 2954 & 34 & 98.8 & 5891 & 55 & 85.0 & 13485 & 79 & 79.2 & 16412 & 94 & 90.7 & 7919 & 61 \\

w/o $S_{ambi}$  & 94.1 & 782 & 38 & 95.8 & 3127 & 39 & 97.5 & 5432 & 48 & 81.7 & 12876 & 72 & 75.0 & 15643 & 88 & 88.8 & 7572 & 57 \\

w/o $S_{dis}$  & 94.0 & 889 & 44 & 95.4 & 2765 & 30 & 95.0 & 6102 & 58 & 80.0 & 14123 & 82 & 70.0 & 16892 & 95 & 86.9 & 8154 & 62 \\

$\mathcal{C}^T(q)$-only  & 93.6 & 945 & 48 & 94.6 & 3541 & 42 & 95.0 & 6834 & 64 & 76.7 & 14987 & 85 & 62.5 & 18231 & 97 & 84.5 & 8907 & 67 \\

 \bottomrule
\end{tabular}
}
\caption{Ablation study results on the (Qwen3-4B-2507-Thinking, Qwen3-4B-2507-Instruct) model pair.}
\label{ablation}
\end{table*}
}

\section{Rating Instruction of DAMI-Prompt}
\label{DAMI-Prompt}

\section*{Ethical Statement}

There are no ethical issues.

\begin{figure*}[h]
\centering
\begin{minipage}{0.95\linewidth}
\begin{lstlisting}[
    basicstyle=\small\ttfamily,
    frame=single,
    backgroundcolor=\color{gray!10},
    breaklines=true,
    columns=fullflexible
]
Rate the difficulty of this math problem from 1 to 10 
using these criteria:

Level 1-2 (Very Easy):
- Simple arithmetic (addition, subtraction, 
  multiplication, division)
- Direct application of basic formulas
- Single-step problems

Level 3-4 (Easy):
- Multi-step arithmetic with 2-3 operations
- Basic word problems with clear solution path
- Simple algebra (solving for one variable)

Level 5-6 (Medium):
- Multi-step problems requiring 4+ operations
- Problems involving ratios, proportions, or percentages
- Systems of equations (2 variables)
- Geometry problems with standard formulas

Level 7-8 (Hard):
- Complex multi-step reasoning
- Abstract mathematical concepts
- Problems requiring creative problem-solving approaches
- Advanced algebra, combinatorics, or number theory

Level 9-10 (Very Hard):
- Competition-level mathematics
- Requires deep mathematical insight or multiple 
  advanced techniques
- Abstract reasoning with non-obvious solution paths

Respond with ONLY a single digit (1-10) based on the 
problem's difficulty.

Problem: {question}

Difficulty rating:
\end{lstlisting}
\end{minipage}
\caption{Prompt template for difficulty rating in DAMI-Prompt.}
\label{fig:rating_prompt}
\end{figure*}

\end{document}